\documentclass[twoside,11pt]{article}

\usepackage{jmlr2e}
\usepackage{bm}
\usepackage{soul}
\usepackage{xcolor}
\usepackage{float}
\usepackage{tikz}
\usetikzlibrary{bayesnet, fit, calc}
\usepackage{enumitem}
\usepackage[export]{adjustbox}
\usepackage{caption}
\usepackage{subcaption}
\usepackage{amsmath}
\usepackage{algorithm}
\usepackage{algpseudocode}
\usepackage{longtable}
\usepackage{array}
\usepackage{booktabs}
\usepackage{multirow}
\usepackage{xparse}

\newcommand{\indep}{\perp \!\!\! \perp}
\newcommand{\dep}{\not\!\perp \!\!\! \perp}
\let\oldReturn\Return
\renewcommand{\Return}{\State\oldReturn}
\makeatletter
\algnewcommand{\LineComment}[1]{\Statex \hskip\ALG@tlm #1}
\def\figcaption{%
     \refstepcounter{figure}%
     \@dblarg{\@caption{figure}}}
\makeatother
\newcolumntype{M}[1]{>{\centering\arraybackslash}m{#1}}
\DeclareMathOperator*{\argmax}{arg\,max}

\newcommand{\fillcol}{red!5}
\newcommand{\bordercol}{red}

\makeatletter
\tikzset{%
     remember picture with id/.style={%
       remember picture,
       overlay,
       save picture id=#1,
     },
     save picture id/.code={%
       \edef\pgf@temp{#1}%
       \immediate\write\pgfutil@auxout{%
         \noexpand\savepointas{\pgf@temp}{\pgfpictureid}}%
     },
     if picture id/.code args={#1#2#3}{%
       \@ifundefined{save@pt@#1}{%
         \pgfkeysalso{#3}%
       }{
         \pgfkeysalso{#2}%
       }
     }
   }

   \def\savepointas#1#2{%
  \expandafter\gdef\csname save@pt@#1\endcsname{#2}%
}

\def\tmk@labeldef#1,#2\@nil{%
  \def\tmk@label{#1}%
  \def\tmk@def{#2}%
}
\tikzdeclarecoordinatesystem{pic}{%
  \pgfutil@in@,{#1}%
  \ifpgfutil@in@%
    \tmk@labeldef#1\@nil
  \else
    \tmk@labeldef#1,(0pt,0pt)\@nil
  \fi
  \@ifundefined{save@pt@\tmk@label}{%
    \tikz@scan@one@point\pgfutil@firstofone\tmk@def
  }{%
  \pgfsys@getposition{\csname save@pt@\tmk@label\endcsname}\save@orig@pic%
  \pgfsys@getposition{\pgfpictureid}\save@this@pic%
  \pgf@process{\pgfpointorigin\save@this@pic}%
  \pgf@xa=\pgf@x
  \pgf@ya=\pgf@y
  \pgf@process{\pgfpointorigin\save@orig@pic}%
  \advance\pgf@x by -\pgf@xa
  \advance\pgf@y by -\pgf@ya
  }%
}

\NewDocumentCommand{\tikzmarkin}{m D(){0.825,-0.10} D(){-0.175,0.27}}{%
      \tikz[remember picture,overlay]
      \draw[line width=1pt,rectangle,fill=\fillcol,draw=\bordercol]
      (pic cs:#1) ++(#2) rectangle (#3)
      ;}

\newcommand\tikzmarkend[2][]{%
\tikz[remember picture with id=#2] #1;}

\makeatletter
\newcommand*{\rom}[1]{\expandafter\@slowromancap\romannumeral #1@}
\makeatother

\ShortHeadings{Improving Bayesian Network Structure Learning in the Presence of Measurement Error}{Liu, Constantinou, and Guo}
\firstpageno{1}
\title{Improving Bayesian Network Structure Learning in the Presence of Measurement Error}

\author{\name Yang Liu\textsuperscript{1}\email yangliu@qmul.ac.uk
\AND
\name Anthony C. Constantinou\textsuperscript{1, 2}\email a.constantinou@qmul.ac.uk
\AND
\name ZhiGao Guo\textsuperscript{1}\email zhigao.guo@qmul.ac.uk\\
\addr\textsuperscript{1}School of Electronic Engineering and Computer Science\\
Queen Mary University of London\\
London, E1 4NS, UK\\
\textsuperscript{2}The Alan Turing Institute\\
London, NW1 2DB, UK}
\editor{}

\begin{document}

\maketitle

\begin{abstract}
Structure learning algorithms that learn the graph of a Bayesian network from observational data often do so by assuming the data correctly reflect the true distribution of the variables. However, this assumption does not hold in the presence of measurement error, which can lead to spurious edges. This is one of the reasons why the synthetic performance of these algorithms often overestimates real-world performance. This paper describes an algorithm that can be added as an additional learning phase at the end of any structure learning algorithm, and serves as a correction learning phase that removes potential false positive edges. The results show that the proposed correction algorithm successfully improves the graphical score of four well-established structure learning algorithms spanning different classes of learning in the presence of measurement error.
\end{abstract}


\section{Introduction}
Bayesian network (BN) is a probabilistic graphic model that captures causal or conditional relationships between variables via a directed acyclic graph (DAG). Learning BNs from observational data is recognised as a challenging problem that has received increasing attention during the past few decades. Various algorithms have been proposed to tackle this problem and are categorised into constraint-based, score-based and hybrid learning algorithms.

The PC algorithm~\citep{spirtes2000causation} is one of the earliest proposed constraint-based algorithms which attempts to recover the Complete Partial Directed Acyclic Graph (CPDAG) of the underlying true causal graph by performing conditional independence tests between variables. Many other algorithms are derived from PC, including MMPC~\citep{tsamardinos2003time} which can handle thousands of variables via sequentially choosing the variable with the maximum association with the target variable into its parents and children set, PC-fdr~\citep{li2009controlling} which controls the false discovery rate of the skeleton of the learned graph under a user-specified level at the limit of large sample sizes and PC-stable~\citep{colombo2014order} which resolves the issue of PC’s output being dependent on the order of variables as they appear in the data. The GES algorithm~\citep{chickering2002optimal}, on the other hand, is a well-established score-based algorithm that searches the optimal CPDAG over two phases. In phase \rom{1}, GES greedily adds edges that maximise the Bayesian score, whereas in phase \rom{2}, it greedily removes edges that maximise the Bayesian score. The ILP algorithm~\citep{cussens2011bayesian} is another well-established algorithm that tackles the structure learning problem with an integer linear programming approach. Lastly, hybrid learning algorithms combine both classes of learning, constraint-based and score-based, and include the MMHC algorithm~\citep{tsamardinos2006max} that combines MMPC with hill-climbing search, and the H2PC algorithm~\citep{gasse2014hybrid} that combines HPC~\citep{gasse2014hybrid} with hill-climbing search.

Most of these algorithms assume that their input data are accurately sampled from the true distributions. However, this assumption is often not true when working with real-world data. The assumption of an underlying measurement error has only recently attracted attention in terms of its effect on BN structure learning. Scheines et al~\citep{scheines2016measurement} studied the effect of Gaussian measurement error on score-based FGES~\citep{ramsey2017million} and showed that even minor levels of measurement error can considerably deteriorate its accuracy. Zhang et al~\citep{zhang2018causal} investigated the linear non-Gaussian models in the presence of measurement error and presented four conditions that make the underlying structure identifiable from the observed variables that incorporate measurement error. Lastly, Blom et al~\citep{blomaupper} proposed a method to estimate the upper bound of the variance of measurement error in linear Gaussian models, and used this bound as a correction of conditional independence tests during constraint-based learning.

Traditionally, measurement error is generated and modelled under the assumption of Normally distributed and continuous data~\citep{bollinger2017bayesian}, although various other types of synthetic noise have recently been investigated with discrete variables~\citep{constantinou2020large}. In this paper, we assume the data are discrete, and that variables with measurement error are children of their underlying error-free version, and not the actual parents of other variables, essentially making them independent of other variables in the graph given their unobserved error-free version. We propose a score-based correction method called Spurious Edge Detection (SED) algorithm which aims to identify and remove potential false positive (FP) edges learned by other structure learning algorithms, often in the presence of measurement error. The remainder of the paper is organised as follows: the terminology and underlying assumptions are described in Section \ref{sec: preliminaries}, Section \ref{sec: impact of measurement error on structure learning} illustrates the impact of measurement error on structure learning, Section \ref{sec: EM-based correction algorithm} describes the correction algorithm, Section \ref{sec: empirical evaluation} presents the results , and we provide our conclusions along with future research directions in Section \ref{sec: conclusion and discussions}.
\section{Preliminaries}
\label{sec: preliminaries}
This section presents the preliminaries and the necessary terminology and assumptions. We assume that each variable present in the data may be subject to measurement error. We refer to variables with measurement error as \textit{noisy} variables and to variables without measurement error as \textit{error-free} variables. We also assume that each potentially noisy variable present in the data is the child of its error-free unobserved version not present in the data. We denote the unobserved error-free variables as $V_i$ where $i$ represents the index of the error-free variable, and its corresponding observed noisy variable as $V_i^o$ where superscript $o$ indicates the observed version of variable $V_i$. We use lowercase letters to represent the assignment of states where $v_i^l$ denotes the $l$th state of variable $V_i$ or corresponding $V_i^o$. We define the error-free graph as $G\left(\bm V, \bm E\right)$ composed of the error-free variable set $\bm V = \left(V_1, \ldots, V_n\right)$ and edge set $\bm E$ between variables $\bm V$. When a variable in $\bm V$ does not incorporate measurement error, we assume that its corresponding error-free $V_i$ and observed $V_i^o$ will share identical distributions; whereas the distributions will differ in the presence of measurement error where the discrepancy between distributions increases with the measurement error.

The assumption that a potentially noisy observed variable $V_i^o$ has only one parent, where this parent represents its unobserved error-free version $V_i$, produces the following \textit{Independence rule}:

\vspace{4mm} 
\textbf{Independence rule}: In the presence of measurement error, an observed variable $V_i^o$ is independent of other variables conditional on its unobserved error-free version $V_i$.
\vspace{4mm} 

Figure \ref{fig: an example of error-free model and observed model} presents a simple example that illustrates the relationship between error-free and observed variables given the Independence rule, where each $V_i^o$ becomes independent of the remaining nodes given its corresponding error-free parent $V_i$. Moreover, if the error-free variable $V_i$ has value $v_i^l$, its corresponding noisy version will be subject to an error rate $\alpha_i^j$ where
\begin{equation}
\label{equ: error function}
\alpha_i^l = 1 - P\left(V_i^o = v_i^l\mid V_i = v_i^l\right)
\end{equation}
In other words, $\alpha_i^l$ represents the rate of observing a value for $V_i^o$ that is not equal to the true value $v_i^l$ of $V_i$. Note that it is possible for different states of $V_i$ to be subject to varying error rates $\alpha_i^l$. We denote the error rate $\alpha_i$ of variable $V_i^o$ in terms of its maximum error rate amongst all states in $V_i$, i.e., $\alpha_i = \max\limits_l\alpha_i^l$.
\begin{figure}[H]
\centering
\captionsetup{format=hang}
  \begin{tikzpicture}[scale = 1]
    \node[latent, dashed] (V^*_1) at (-1.5, 0) {$V_1$};
    \node[latent] (V_1) at (-1.5, -1.5) {$V^o_1$};
    \node[latent, dashed] (V^*_2) at (0, 0) {$V_2$};
    \node[latent] (V_2) at (0, -1.5) {$V^o_2$};
    \node[latent, dashed] (V^*_3) at (1.5, 0) {$V_3$};
    \node[latent] (V_3) at (1.5, -1.5) {$V^o_3$};
    \node[] at (3.3, 0) {$G\left(\bm V, \bm E\right)$};

    \edge {V^*_1} {V^*_2};
    \edge {V^*_2} {V^*_3};
    \edge {V^*_1} {V_1};
    \edge {V^*_2} {V_2};
    \edge {V^*_3} {V_3};
    \draw[thick, dashed] (-2, -0.5) rectangle (2, 0.5);
  \end{tikzpicture}
  \caption{A hypothetical graph that illustrates the relationship between the error-free variables $\bm V$ and the corresponding observed variables $\bm V^o$ given the Independence rule, where a noisy variable $V_i^o$ becomes independent of other variables in $\bm G$ given $V_i$}
  \label{fig: an example of error-free model and observed model}
\end{figure}
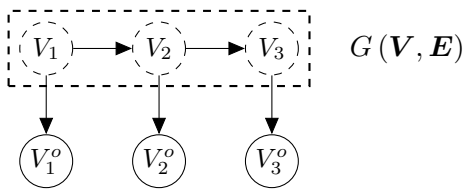
This paper also adopts the following widely used assumptions~\citep{spirtes2000causation}:
\begin{enumerate}[label = (\roman*)]
    \item\textbf{Markov assumption}: Given a directed acyclic graph $G$ over a variable set $\bm V$, every variable in $\bm V$ is independent of its non-descendants conditional on its parents.
    \item\textbf{Causal Faithfulness assumption}: Given a directed acyclic graph $G$ over a variable set $\bm V$, a probability distribution $P\left(\bm V\right)$ is faithful to $G$ if and only if the conditional independence relationships in $P\left(\bm V\right)$ are exactly the same as the independence relationships inferred by \textit{d-separation criterion}~\citep{spirtes2000causation} from $G$.
    \item \textbf{Causal Sufficiency assumption}: There are no unmeasured variables acting as a common cause of any two or more observed variables.
\end{enumerate}
\section{The impact of measurement error on structure learning}
\label{sec: impact of measurement error on structure learning}
This section illustrates that measurement error generally causes the structure learning algorithms to produce a higher number of spurious edges that tend to lead to a greater number of 3-vertex cliques, compared to the true number of such cliques in the ground truth graph. A clique is a set of nodes where each pair of nodes in the clique is adjacent. We first explain why this phenomenon occurs in theory, from the perspective of constraint-based learning, and then present the effect in practise by illustrating the empirical effect of measurement error on algorithms spanning all three classes of learning. Because constraint-based learning relies on statistical tests, we discuss the effect of measurement error in terms of both marginal and conditional dependencies between variables. Given the Causal Faithfulness assumption, the dependencies between variables are consistent with those entailed by applying d-separation rules on the BN. Therefore, we restrict the description about the effect of measurement error on d-connections and d-separations. For the unconditional (i.e., marginal dependence) case, we derive the Theorem{~\ref{the: unconditional}}.
\begin{theorem}
\label{the: unconditional}
The d-connection and d-separation relationships between two error-free variables $V_1$ and $V_2$ in an error-free graph $G$ are consistent with the d-connection and d-separation relationships of their corresponding observed versions $V_1^o$ and $V_2^o$ affected by measurement error, given the Independence rule.
\end{theorem}
\begin{proof}
    \begin{enumerate}
        \item When $V_1$ and $V_2$ are d-separated, this implies that there is either no direct path or no indirect unblocked path between $V_1$ and $V_2$ in $G$. Given Independence rule, the only neighbours of $V_1^o$ and $V_2^o$ are $V_1$ and $V_2$ who serve as their respective error-free parents. Thus, there is also either no direct path and no indirect unblocked path between $V_1^o$ and $V_2^o$ which means $V_1^o$ and $V_2^o$ are also d-separated.
        \item When $V_1$ and $V_2$ are d-connected, there must be at least one unblocked path $p$ from $V_1$ to $V_2$. Given Independence rule, $V_1$ and $V_2$ are the respective parents of $V_1^o$ and $V_2^o$. Thus, by combining $V_1^o\leftarrow V_1, p\textrm{ and }V_2\rightarrow V_2^o$, we can find an unblocked path from $V_1^o$ to $V_2^o$ that makes $V_1^o$ and $V_2^o$ d-connected.
    \end{enumerate}
\end{proof}
According to Theorem~\ref{the: unconditional}, the unconditional relationship between error-free variables should be consistent with the unconditional relationship of their corresponding noisy observed variables given the Causal Faithfulness assumption. However, the conditional independence between error-free variables may not always hold for their corresponding noisy observed versions. Figure~\ref{fig: the effect of measurement error on two different kinds of connection} illustrates two different causal classes with measurement error on the node $S^o$. Specifically, Figure~\ref{fig: the effect of measurement error on two different kinds of connection (a)} represents the causal class of common-effect where $V_1$ and $V_2$ become d-connected conditional on either $S$ or its noisy version $S^o$, whereas Figure~\ref{fig: the effect of measurement error on two different kinds of connection (b)} represents the causal class of causal-chain where $V_1$ and $V_2$ become d-separated conditional on $S$, yet they remain d-connected conditional on $S^o$ (this observation also holds for the causal class of common-cause).
\begin{figure}[H]
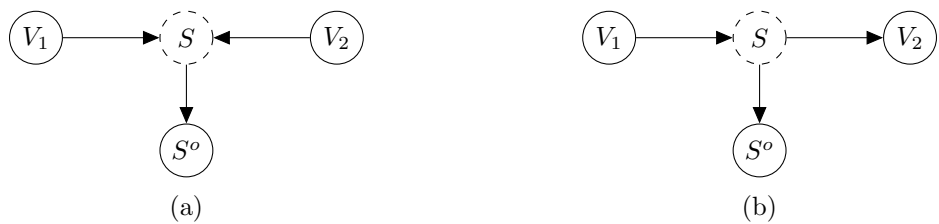
 
\centering
\captionsetup{format=hang}
\begin{adjustbox}{minipage = \linewidth, scale = 1}
\begin{subfigure}{.5\textwidth}
  \centering
  \tikz{
    \node[latent, dashed] (S_i^*) at (0, 0) {$S$};
    \node[latent] (V_2^*) at (2, 0) {$V_2$};
    \node[latent] (S_i) at (0, -1.5) {$S^o$};
    \node[latent] (V_1^*) at (-2, 0) {$V_1$};

    \edge {S_i^*} {S_i};
    \edge {V_1^*} {S_i^*};
    \edge {V_2^*} {S_i^*};
  }
  \caption{}
  \label{fig: the effect of measurement error on two different kinds of connection (a)}
\end{subfigure}%
\begin{subfigure}{.5\textwidth}
  \centering
  \tikz{
    \node[latent, dashed] (S_i^*) at (0, 0) {$S$};
    \node[latent] (V_2^*) at (2, 0) {$V_2$};
    \node[latent] (S_i) at (0, -1.5) {$S^o$};
    \node[latent] (V_1^*) at (-2, 0) {$V_1$};

    \edge {S_i^*} {S_i};
    \edge {V_1^*} {S_i^*};
    \edge {S_i^*} {V_2^*};
  }
  \caption{}
  \label{fig: the effect of measurement error on two different kinds of connection (b)}
\end{subfigure}
\end{adjustbox}
\caption{Modelling the presence of measurement error on the two different causal equivalence classes where case (a) represents the common-effect class, where $V_1$ and $V_2$ become d-connected conditional on either $S$ or $S^o$, and (b) represents the causal-chain class where $V_1$ and $V_2$ become d-separated conditional on $S$, although they remain d-connected conditional on noisy $S^o$ (this also holds for the causal class of common-cause).}
\label{fig: the effect of measurement error on two different kinds of connection}
\end{figure}
These lead to Theorem{~\ref{the: conditional dependent}} and{~\ref{the: conditional independent}} which state that although the conditional d-connection relation is consistent between error-free variables and observed noisy variables, it is likely that some conditional d-separations will not hold when the observed variables incorporate measurement error.
\begin{theorem}
\label{the: conditional dependent}
If two error-free variables $V_1$ and $V_2$ are d-connected given a variable set $\bm S$, this d-connection will also hold for their observed noisy variables $V_1^o$ and $V_2^o$ conditional on noisy variable set $\bm S^o$.
\end{theorem}
\begin{proof}
If $V_1$ and $V_2$ are d-connected given $\bm S$, there must be an unblocked path $p$ between $V_1$ and $V_2$ conditional on $\bm S$. Since $p$ also remains unblocked given $\bm S^o$, $V_1$ and $V_2$ remain d-connected given $\bm S^o$. Thus, $V_1^o$ and $V_2^o$ are also d-connected given $\bm S^o$.
\end{proof}
\begin{theorem}
\label{the: conditional independent}
If two error-free variables $V_1$ and $V_2$ that are unconditionally d-connected become d-separated conditional on a variable set $\bm S$ that contains error-free variables, then the observed noisy variables $V_1^o$ and $V_2^o$ will not be d-separated conditional on the observed noisy variable set $\bm S^o$.
\end{theorem}
\begin{proof}
When $V_1$ and $V_2$ that are unconditionally d-connected become d-separated conditional on $\bm S$, there must be at least one path $p$ between $V_1$ and $V_2$ that d-separates them conditional on $\bm S$ and this specific path cannot contain common-effect (converging) connections. Therefore, a path $p$ composed of a series of non-converging connections will not d-separate $V_1$ and $V_2$ given a noisy variable set $S^o$ (as shown in Figure~\ref{fig: the effect of measurement error on two different kinds of connection (b)}). Since $V_1$ and $V_2$ are the only parent of $V_1^o$ and $V_2^o$ respectively, $V_1^o$ and $V_2^o$ will also not be d-separated conditional on $\bm S^o$.
\end{proof}
\begin{figure}[H]
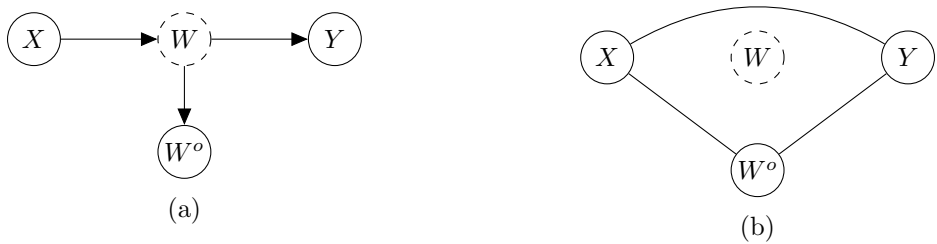
 
\centering
\captionsetup{format=hang}
\begin{adjustbox}{minipage = \linewidth, scale = 1}
\begin{subfigure}{.5\textwidth}
  \centering
  \tikz{
    \node[latent, dashed] (W^*) at (0, 0) {$W$};
    \node[latent] (Y^*) at (2, 0) {$Y$};
    \node[latent] (W) at (0, -1.5) {$W^o$};
    \node[latent] (X^*) at (-2, 0) {$X$};

    \edge {W^*} {W};
    \edge {X^*} {W^*};
    \edge {W^*} {Y^*};
  }
  \caption{}
  \label{fig: the effect of measurement error on constrained-based algorithm (a)}
\end{subfigure}%
\begin{subfigure}{.5\textwidth}
  \centering
  \tikz{
    \node[latent, dashed] (W^*) at (0, 0) {$W$};
    \node[latent] (Y^*) at (2, 0) {$Y$};
    \node[latent] (W) at (0, -1.5) {$W^o$};
    \node[latent] (X^*) at (-2, 0) {$X$};

    \path [-] (X^*) edge [bend left] (Y^*);
    \path [-] (X^*) edge (W);
    \path [-] (W) edge (Y^*);
  }
  \caption{}
  \label{fig: the effect of measurement error on constrained-based algorithm (b)}
\end{subfigure}
\end{adjustbox}
\caption{(a) A BN containing error-free variables $X$ and $Y$, and variable $W$ whose observations are drawn from its noisy version $W^o$ due to the presence of measurement error. (b) The graph learned by applying constraint-based learning to the observed data sampled from $X, Y$ and $W^o$.}
\label{fig: the effect of measurement error on constrained-based algorithm}
\end{figure}
Next, let us consider the impact of measurement error on constraint-based learning. The starting point of algorithms in this class is a fully connected undirected graph. Edges between variables are then removed if any marginal or conditional independence between the two variables are discovered. Consider the simple BN shown in Figure~\ref{fig: the effect of measurement error on constrained-based algorithm (a)} composed by three variables $X, Y$ and $W$, where $X$ and $Y$ are error-free whereas $W$ incorporates measurement error; implying that observations on $W$ are drawn from its noisy version $W^o$.

According to Theorems~\ref{the: unconditional} and~\ref{the: conditional dependent}, and with reference to the example in Figure~\ref{fig: the effect of measurement error on constrained-based algorithm}, the unconditional and conditional dependences between error-free variables $X$ and $W$ extent to their observed versions. Therefore constraint-based learning produces an edge between $X$ and $W^o$ in the graph learned from observed noisy data (and similarly for $W^o$ and $Y$). The only conditional independence relationship amongst the error-free variables is $X\indep Y\mid W$. According to Theorem~\ref{the: conditional independent}, this conditional independence does not hold in the presence of measurement error on $W^o$. Therefore, we get $X\dep Y\mid W^o$ and the incorrect fully connected graph shown in Figure~\ref{fig: the effect of measurement error on constrained-based algorithm (b)} as the learned graph. In other words, the measurement error on an unshielded non-collider misleads constraint-based learning towards a spurious edge between its neighbours, producing a 3-vertex clique. Note that constraint-based learning can reconstruct $X-W-Y$ when the input data does not incorporate measurement error.

We, therefore consider a 3-vertex clique as a signal for the presence of measurement error in at least one of the variables that make up the clique. When a learned graph contains such a clique, we need to determine whether the clique exists in the error-free graph or whether it is the result of measurement error. If we could distinguish between these two possibilities, then we could recover the error-free model from noisy data. This challenge can be viewed as a type of a hidden variable problem. In our case, a potential hidden variable represents the error-free parent of its corresponding observed and potentially noisy version.

While, in practice, Theorems~\ref{the: unconditional},~\ref{the: conditional dependent} and~\ref{the: conditional independent} will not hold for all statistical tests used to explore the d-connection and d-separation scenarios discussed above, they can still help us identify graphical inaccuracies that due to measurement error. The level of accuracy in determining such inaccuracies may critically depend on the rate of error, how it differs per state of a variable, and how it relates to distributional errors in other variables. Figure~\ref{fig: CPDAG learned by PC algorithm from Asia data set} presents an example based on the PC-Stable algorithm and the classic Asia network, with synthetic data of sample size 10,000. Specifically, Figure~\ref{fig: true Asia network} represents the ground true graph, Figure~\ref{fig: cpdag_ef} the learned error-free graph, and Figure \ref{fig: cpdag_me} the learned graph with 5\% measurement error on variable $bronc$, as defined by Equation~\ref{equ: error function}. This relatively small rate of error has led to the spurious edge between $smoke$ and $dysp$. This is because while $smoke$ and $dysp$ are independent conditional on the error-free variables $bronc$ and $either$, this conditional independence is relaxed in the presence of measurement error on variable $bronc$ and hence, the algorithm produces the additional FP edge. Moreover, this additional edge produces the 3-vertex clique $\left\{smoke, bronc, dysp\right\}$ that does not exist in the true graph nor in the error-free learned graph.
\begin{figure}[H]
\centering
\captionsetup{format=hang}
\begin{adjustbox}{minipage = \linewidth, scale = 1}
\begin{subfigure}{.33\textwidth}
  \centering
    \begin{tikzpicture}[scale = 0.8]
        \node[latent, minimum size = 1cm] (asia) at (0, 0) {asia};
        \node[latent, minimum size = 1cm] (tub) at (0, -2) {tub};
        \node[latent, minimum size = 1cm] (either) at (1.25, -4) {either};
        \node[latent, minimum size = 1cm] (lung) at (2.5, -2) {lung};
        \node[latent, minimum size = 1cm] (smoke) at (4, 0) {smoke};
        \node[latent, minimum size = 1cm] (bronc) at (4, -3) {bronc};
        \node[latent, minimum size = 1cm] (dysp) at (4, -6) {dysp};
        \node[latent, minimum size = 1cm] (xray) at (1.25, -6) {xray};
        
        \draw[->] (asia) edge (tub);
        \draw[->] (smoke) edge (lung);
        \draw[->] (smoke) edge (bronc);
        \draw[->] (bronc) edge (dysp);
        \draw[->] (tub) edge (either);
        \draw[->] (lung) edge (either);
        \draw[->] (either) edge (xray);
        \draw[->] (either) edge (dysp);
    \end{tikzpicture}
  \caption{}
  \label{fig: true Asia network}
\end{subfigure}%
\begin{subfigure}{.33\textwidth}
  \centering
    \begin{tikzpicture}[scale = 0.8]
        \node[latent, minimum size = 1cm] (asia) at (0, 0) {asia};
        \node[latent, minimum size = 1cm] (tub) at (0, -2) {tub};
        \node[latent, minimum size = 1cm] (either) at (1.25, -4) {either};
        \node[latent, minimum size = 1cm] (lung) at (2.5, -2) {lung};
        \node[latent, minimum size = 1cm] (smoke) at (4, 0) {smoke};
        \node[latent, minimum size = 1cm] (bronc) at (4, -3) {bronc};
        \node[latent, minimum size = 1cm] (dysp) at (4, -6) {dysp};
        \node[latent, minimum size = 1cm] (xray) at (1.25, -6) {xray};
        
        \draw[-] (asia) edge (tub);
        \draw[-] (smoke) edge (lung);
        \draw[-] (smoke) edge (bronc);
        \draw[-] (bronc) edge (dysp);
        \draw[->] (tub) edge (either);
        \draw[->] (lung) edge (either);
    \end{tikzpicture}
  \caption{}
  \label{fig: cpdag_ef}
\end{subfigure}%
\begin{subfigure}{.33\textwidth}
  \centering
  \begin{tikzpicture}[scale = 0.8]
        \node[latent, minimum size = 1cm] (asia) at (0, 0) {asia};
        \node[latent, minimum size = 1cm] (tub) at (0, -2) {tub};
        \node[latent, minimum size = 1cm] (either) at (1.25, -4) {either};
        \node[latent, minimum size = 1cm] (lung) at (2.5, -2) {lung};
        \node[latent, minimum size = 1cm] (smoke) at (4, 0) {smoke};
        \node[latent, minimum size = 1cm] (bronc) at (4, -3) {bronc};
        \node[latent, minimum size = 1cm] (dysp) at (4, -6) {dysp};
        \node[latent, minimum size = 1cm] (xray) at (1.25, -6) {xray};
        
        \draw[-] (asia) edge (tub);
        \draw[-] (smoke) edge (lung);
        \draw[-] (smoke) edge (bronc);
        \draw[-] (bronc) edge (dysp);
        \draw[->] (tub) edge (either);
        \draw[->] (lung) edge (either);
        \draw[-] (smoke) edge [bend left] (dysp);
    \end{tikzpicture}
  \caption{}
  \label{fig: cpdag_me}
\end{subfigure}
\end{adjustbox}
\caption{(a) The true Asia network. (b) The CPDAG learned by PC-Stable given the error-free synthetic data set. (c) The CPDAG learned by PC-Stable given the same synthetic data set but with 5\% measurement error on variable $bronc$.}
\label{fig: CPDAG learned by PC algorithm from Asia data set}
\end{figure}
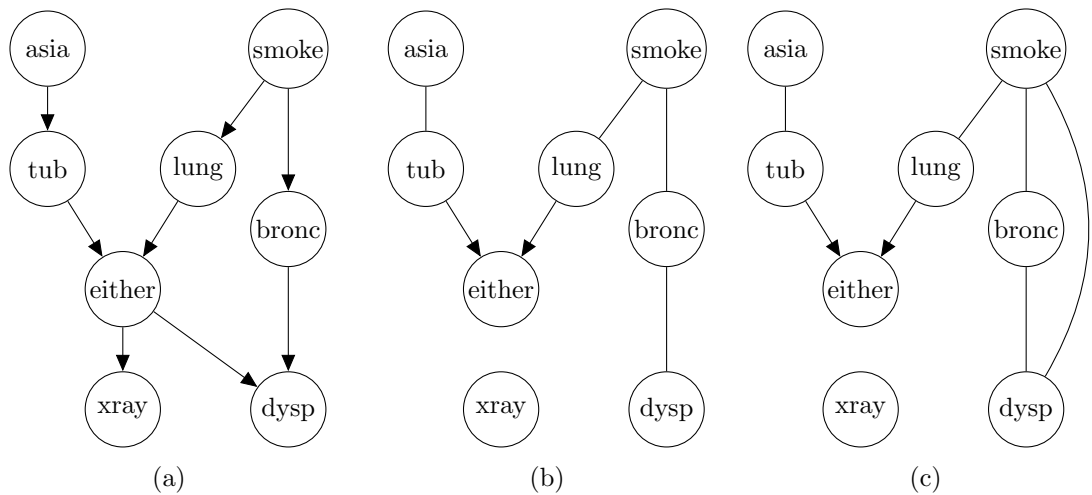
To investigate the impact of measurement error on BN structure learning in general, we have extended these experiments to four algorithms spanning different classes of learning. Namely, in addition to constraint-based PC-stable~\citep{colombo2014order}, to the score-based HC~\citep{bouckaert1994properties} and ILP~\citep{cussens2011bayesian}, and to hybrid H2PC~\citep{gasse2014hybrid}. We have used each of these algorithms to reconstruct 50 randomly generated BNs consisting of 20 Boolean nodes, using the method described in~\citep{ide2002random}. Each random network was used to generate two synthetic data sets of 10,000 sample size each; one error-free data set and another noisy data set with 10\% measurement error on each variable.

Figure~\ref{fig: average number of cliques in learned graphs} compares the average number of 3-vertex cliques produced by each of the algorithms with and without measurement error, and with reference to the average number of 3-vertex cliques present in the ground truth graphs. These initial results show that score-based learning is more sensitive to the measurement error compared to constraint-based learning, and this naturally extends to hybrid learning. These results support our hypothesis that a 3-vertex clique can be viewed as a signal for the presence of measurement error in the input data.
\begin{figure}[H]
    \centering
    \captionsetup{format=hang}
    \includegraphics[width=\linewidth]{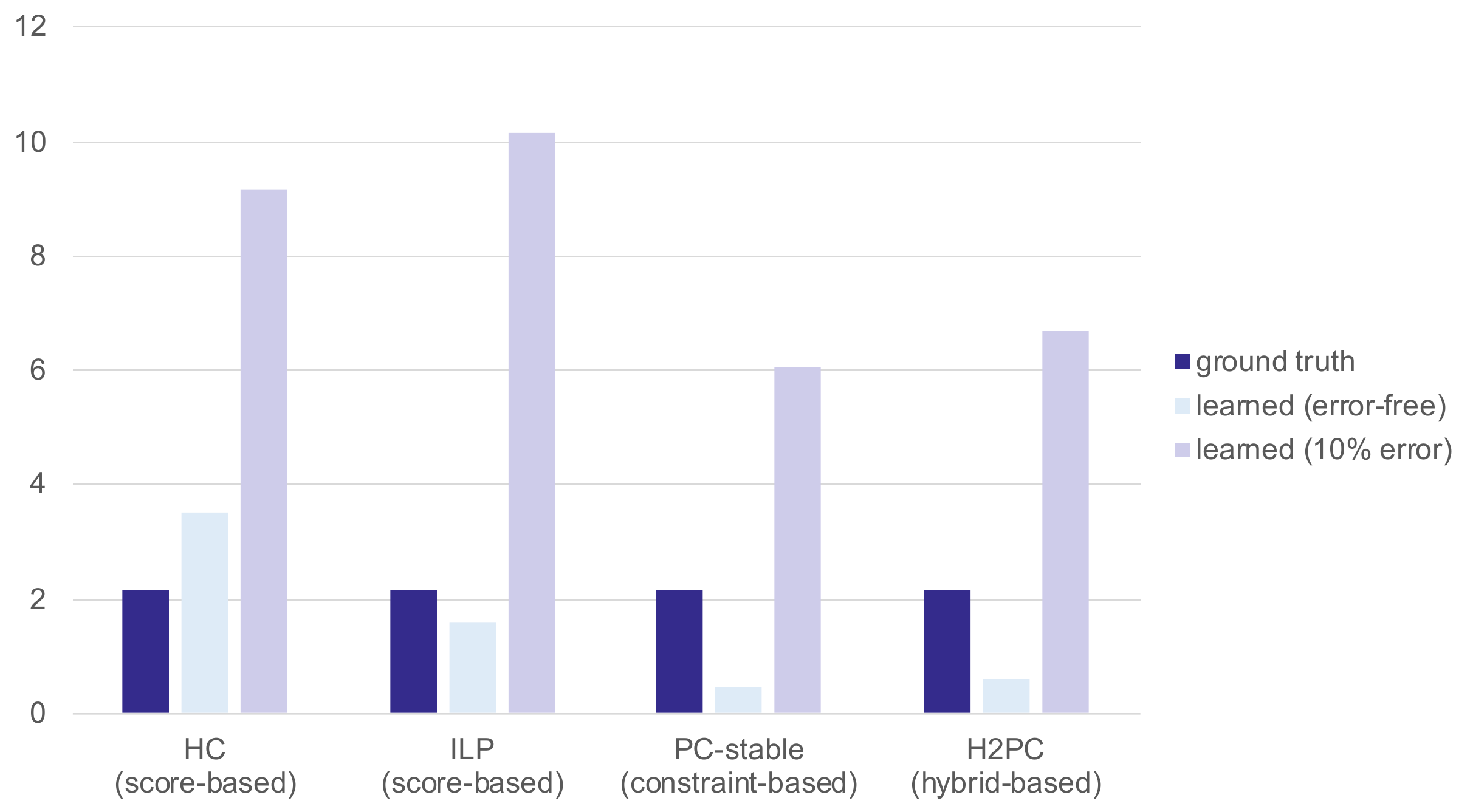}
    \caption{The average number of 3-vertex cliques in the ground truth graphs, the graphs learned from error-free data sets, and the graphs learned from observed data sets with 10\% measurement error on each variable.}
    \label{fig: average number of cliques in learned graphs}
\end{figure}
\section{The Spurious Edge Detection (SED) algorithm}
\label{sec: EM-based correction algorithm}
This section describes the Spurious Edge Detection (SED) algorithm which can be applied to the output graph produced by any other BN structure learning algorithms to discover and eliminate potential FP edges that tend to be the outcome of measurement error. The implementation of SED is available online~\footnote{Our code is publicly available at~\url{https://github.com/Enderlogic/Spurious-Edge-Detection}.}. Further to what has been discussed in Section~\ref{sec: impact of measurement error on structure learning}, SED focuses its search for FP edges on the induced subgraph of 3-vertex cliques and assumes that one of the three edges in such an induced subgraph may be a FP.

We define the Candidate Spurious Edge set $CSE\left(V_i\right)$ for a candidate noisy variable $V_i$ as the set of edges between neighbours of $V_i$, since the existence of these edges might be due to measurement error on $V_i$ (refer to the discussion of Figure{~\ref{fig: the effect of measurement error on constrained-based algorithm}}). The complete $CSE$ contains $CSE\left(V_i\right)$ for all $V_i$ in $G$, i.e., $CSE = \left\{V_i: CSE(V_i)\mid \textrm{for all }V_i\textrm{ in }G\right\}$. For instance, the $CSE$ sets for each variable in Figure{~\ref{fig: CSE example}} are:
\begin{equation*}
    \begin{split}
        CSE\left(A\right): & \left\{B\rightarrow C, B\rightarrow E\right\}\\
        CSE\left(B\right): & \left\{A\rightarrow C, A\rightarrow E, C\rightarrow D\right\}\\
        CSE\left(C\right): & \left\{A\rightarrow B, B\rightarrow D\right\}\\
        CSE\left(D\right): & \left\{B\rightarrow C\right\}\\
        CSE\left(E\right): & \left\{A\rightarrow B\right\}
    \end{split}
\end{equation*}
\begin{figure}[H]
    \centering
    \captionsetup{format=hang}
    \begin{tikzpicture}[scale = 0.8]
        \node[latent] (A) at (0, 0) {A};
        \node[latent] (B) at (2, -1) {B};
        \node[latent] (C) at (4, 0) {C};
        \node[latent] (D) at (4, -2) {D};
        \node[latent] (E) at (0, -2) {E};
        
        \draw[->] (A) edge (B) edge (C) edge (E);
        \draw[->] (B) edge (C) edge (D) edge (E);
        \draw[->] (C) edge (D);
    \end{tikzpicture}
    \caption{An example of a graph that contains multiple 3-vertex cliques}
    \label{fig: CSE example}
\end{figure}
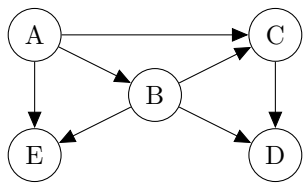
Next, let us revisit the Asia network example in Figure~\ref{fig: cpdag_me} to investigate the possibility of a spurious edge in the presence of a single 3-vertex clique in the learned graph. Recall this is the graph learned by PC-Stable in the presence of 5\% measurement error on variable $bronc$. Since the graph contains a single 3-vertex clique, the $CSE$ sets for each variable of this graph are: 
\begin{equation*}
    \begin{split}
        CSE\left(smoke\right): \left\{bronc-dysp\right\}\\
        CSE\left(bronc\right): \left\{smoke-dysp\right\}\\
        CSE\left(dysp\right): \left\{smoke-bronc\right\}
    \end{split}
\end{equation*}
We then perform three graphical reconstructions, one for each edge in $CSE$ given the Independence rule defined in Section~\ref{sec: preliminaries} to identify and eliminate spurious edges. This process is described in Algorithm~\ref{alg: reconstructing} with inputs a learned graph $G$, a candidate noisy variable $V$, an edge set $E$ contaning candidate spurious edges, and a data set $D$. The output $\Delta$ of Algorithm{~\ref{alg: reconstructing}} represents the difference in BIC score between the reconstructed graph produced given $E$ and $V$, and the input learned graph $G$. Figure~\ref{fig: reconstructed graphs} presents the three reconstructed graphs produced by Algorithm{~\ref{alg: reconstructing}} after setting $G$ as Figure{~\ref{fig: cpdag_me}}, $E$ as a set with one candidate edge in $CSE$ and $V$ as the corresponding noisy variable of $E$. The hidden variable (dashed node) in Figure{~\ref{fig: reconstructed graphs}} represents the unmeasured error-free parent of the candidate noisy variable as described in Algorithm~\ref{alg: reconstructing}. Each of the three reconstructed graphs can be seen as a potential ground truth graph that could explain the suspected spurious edge in $CSE$ under assessment.
\begin{algorithm}[H]
\caption{Graph reconstruction procedure}
\label{alg: reconstructing}
\begin{algorithmic}[1]
\Procedure{Reconstruction}{$G, V, E, D$}

\underline{Input}: graph $G$, variable $V$, edge set $E$, data $D$

\underline{Output}: difference in BIC score between reconstructed graph and input graph $\Delta$
\State Compute the BIC score $score_i$ of the input graph $G$
\State Create a copy of graph $G$ in $G_{r}$
\State Remove edges $E$ in $G_{r}$
\State \parbox[t]{\dimexpr\textwidth-\leftmargin-\labelsep-\labelwidth}{%
Replace the observed variable $V$ in $G_{r}$ with a hidden variable that preserves the state space of $V$\strut}
\State \parbox[t]{\dimexpr\textwidth-\leftmargin-\labelsep-\labelwidth}{%
Reintroduce the observed variable $V$ as the observed and suspected noisy variable $V^o$ in $G_{r}$, as the child of the hidden error-free variable $V$\strut}
\State Compute the BIC score $score_r$ of the reconstructed graph $G_r$
\State $\Delta = score_r - score_i$
\Return $\Delta$
\EndProcedure
\end{algorithmic}
\end{algorithm}
\begin{figure}[H]
\centering
\captionsetup{format=hang}
\begin{adjustbox}{minipage = \linewidth, scale = 0.86}
\begin{subfigure}{.33\textwidth}
  \centering
    \begin{tikzpicture}[scale = 0.8]
                \draw node[latent, minimum size = 1cm] (asia) at (0, 0) {asia};
                \draw node[latent, minimum size = 1cm] (tub) at (0, -2) {tub};
                \draw node[latent, minimum size = 1cm] (either) at (1.25, -4) {either};
                \draw node[latent, minimum size = 1cm] (lung) at (2.25, -2) {lung};
                \draw node[latent, minimum size = 1cm, dashed] (smoke*) at (4, 0) {smoke};
                \draw node[latent, minimum size = 1cm, scale = 0.87] (smoke) at (2, 0) {$\textrm{smoke}^o$};
                \draw node[latent, minimum size = 1cm] (bronc) at (4, -3) {bronc};
                \draw node[latent, minimum size = 1cm] (dysp) at (4, -6) {dysp};
                \draw node[latent, minimum size = 1cm] (xray) at (1.25, -6) {xray};
                
                \draw[-] (asia) edge (tub);
                \draw[-] (smoke*) edge (lung);
                \draw[-] (smoke*) edge (bronc);
                \draw[->] (tub) edge (either);
                \draw[->] (lung) edge (either);
                \draw[->] (smoke*) edge (smoke);
                \draw[-] (smoke*) edge [bend left] (dysp);
            \end{tikzpicture}
  \caption{}
  \label{fig: reconstructed (a)}
\end{subfigure}%
\begin{subfigure}{.33\textwidth}
  \centering
    \begin{tikzpicture}[scale = 0.8]
                \draw node[latent, minimum size = 1cm] (asia) at (0, 0) {asia};
                \draw node[latent, minimum size = 1cm] (tub) at (0, -2) {tub};
                \draw node[latent, minimum size = 1cm] (either) at (1, -4) {either};
                \draw node[latent, minimum size = 1cm] (lung) at (2.5, -2) {lung};
                \draw node[latent, minimum size = 1cm] (smoke) at (4, 0) {smoke};
                \draw node[latent, minimum size = 1cm, dashed] (bronc*) at (4, -3) {bronc};
                \draw node[latent, minimum size = 1cm, scale = 0.9] (bronc) at (5.5, -4.5) {$\textrm{bronc}^o$};
                \draw node[latent, minimum size = 1cm] (dysp) at (4, -6) {dysp};
                \draw node[latent, minimum size = 1cm] (xray) at (1, -6) {xray};
                
                \draw[-] (asia) edge (tub);
                \draw[-] (smoke) edge (lung);
                \draw[-] (smoke) edge (bronc*);
                \draw[-] (bronc*) edge (dysp);
                \draw[->] (tub) edge (either);
                \draw[->] (lung) edge (either);
                \draw[->] (bronc*) edge (bronc);
            \end{tikzpicture}
  \caption{}
  \label{fig: reconstructed (b)}
\end{subfigure}%
\begin{subfigure}{.33\textwidth}
  \centering
    \begin{tikzpicture}[scale = 0.8]
                \draw node[latent, minimum size = 1cm] (asia) at (0, 0) {asia};
                \draw node[latent, minimum size = 1cm] (tub) at (0, -2) {tub};
                \draw node[latent, minimum size = 1cm] (either) at (1, -4) {either};
                \draw node[latent, minimum size = 1cm] (lung) at (2.5, -2) {lung};
                \draw node[latent, minimum size = 1cm] (smoke) at (4, 0) {smoke};
                \draw node[latent, minimum size = 1cm] (bronc) at (4, -3) {bronc};
                \draw node[latent, minimum size = 1cm, dashed] (dysp*) at (4, -6) {dysp};
                \draw node[latent, minimum size = 1cm] (dysp) at (6, -4) {$\textrm{dysp}^o$};
                \draw node[latent, minimum size = 1cm] (xray) at (1, -6) {xray};
                
                \draw[-, thick] (asia) edge (tub);
                \draw[-, thick] (smoke) edge (lung);
                \draw[-, thick] (bronc) edge (dysp*);
                \draw[->, thick] (tub) edge (either);
                \draw[->, thick] (lung) edge (either);
                \draw[-, thick] (smoke) edge [bend left] (dysp*);
                \draw[->, thick] (dysp*) edge (dysp);
            \end{tikzpicture}
  \caption{}
  \label{fig: reconstructed (c)}
\end{subfigure}
\end{adjustbox}
\caption{{The three reconstructed graphs for clique $\left\{bronc, dysp, smoke\right\}$, based on the learned graph in Figure~\ref{fig: cpdag_me}. Dotted nodes represent possible hidden error-free parents of the suspected noisy node under assessment.}}
\label{fig: reconstructed graphs}
\end{figure}
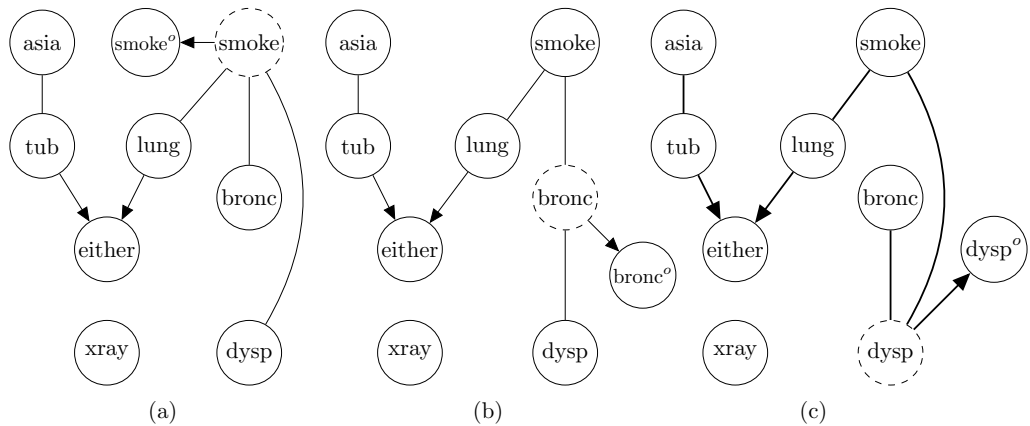
For example, the graph in Figure~\ref{fig: reconstructed (a)} investigates the possibility of variable $smoke$ incorporating measurement error, which is why it is replaced with a hidden unmeasured variable representing its error-free version, with the observed version $smoke^o$ restructured as a child of the hidden variable. Moreover, the edge between $bronc$ and $dysp$ is removed since a possible conditional independence $bronc\indep dysp\mid smoke$ will not hold if variable $smoke$ is indeed noisy, which could explain the presence of clique $\left\{bronc, dysp, smoke\right\}$ in the learned graph shown in Figure~\ref{fig: cpdag_me}. Similarly, Figures~\ref{fig: reconstructed (b)} and~\ref{fig: reconstructed (c)} repeat this process for the remaining two variables in clique $\left\{bronc, dysp, smoke\right\}$.

Each reconstructed graph is evaluated in terms of model selection between the learned and observed distributions using the Bayesian Information Criterion (BIC). Because the reconstructed graphs include an additional hidden variable, we adopt the Expectation-Maximization (EM) learning~\citep{dempster1977maximum} to compute the Log-Likelihood (LL) score of the BIC for each of the reconstructed graphs, and we describe this process in Appendix~\ref{app: EM algorithm adn BIC score}. If at least one reconstructed graph produces a BIC score that is higher than the BIC score produced by the original learned graph, we assume the reconstructed graph with the highest score is a more accurate representation of the underlying ground truth graph, and on this basis we eliminate the spurious edge not present in the optimal reconstructed graph from the original learned graph; otherwise, no modification is made to the original learned graph.

The above example illustrates the process of identifying possible spurious edges when the learned graph contains a single 3-vertex clique. When the learned graph contains multiple such cliques, then it becomes possible that multiple observed variables incorporate measurement error. An optimal assessment of this scenario requires searching for ground truth graph that incorporates multiple hidden error-free variables. However, because the complexity of EM learning grows exponentially with the number of hidden variables~\citep{tembo2016tutorial}, this procedure quickly becomes intractable. We, therefore, iterate through reconstructed graphs that contain a single hidden variable at a time using a heuristic algorithm which we call SED (Spurious Edge Detection). The pseudocode of the SED algorithm is presented in Algorithm~\ref{alg: SED algorithm}.

The SED algorithm is an iterative process that searches for spurious edges by recursively executing two phases, and produces a modified graph that does not contain the edges identified as FP edges. The first phase involves searching for the most likely spurious edge and its corresponding noisy variable amongst all candidate edges, whereas the second phase involves searching for further possible spurious edges that associate with the noisy variable selected in the first phase. Specifically, SED initialises the modified graph $G_{mod}$ from the original learned graph $G$ taken as an input, and generates the $CSE$ from $G$. SED then initiates Phase 1 by applying the graph reconstruction procedure (Algorithm~\ref{alg: reconstructing}) on every edge and its corresponding variable in $CSE$. Note that the same edge may appear more than once in multiple $CSE\left\{V_i\right\}$, implying that Phase 1 assesses the same edge under each of those conditions that could explain the edge as a spurious edge. If at least one edge $E$ in $CSE$ is found to improve BIC compared to the BIC of the input graph $G$, then SED removes the edge $E_m$ in $G_{mod}$ that produces the highest BIC and records its corresponding variable $V_m$ as noisy. To avoid considering multiple spurious edges that are from the induced subgraph of the same 3-vertex clique, the $CSE$ set is also updated given the current modified graph $G_{mod}$. If $CSE\left(V_m\right)$ is not empty, then SED enters Phase 2 by searching for further potential spurious edges caused by noisy $V_m$.

The exploration in Phase 2 continues to be based on the noisy variable recorded in Phase 1. Phase 2 attributes the existence of multiple candidate spurious edges with the same noisy variable to measurement error, to avoid simulating multiple hidden error-free variables in the reconstructed graph. Specifically, Phase 2 initialises an edge set $E_d$ that holds the spurious edge $E_m$ discovered in Phase 1, along with a threshold $\Delta_{MAX}$ which represents the highest output of the reconstruction procedure in Phase 1; i.e., the difference in BIC score between the optimal reconstructed graph and the input learned graph. Then, Phase 2 repeats the reconstruction procedure on each edge $E$ in $CSE\left(V_m\right)$ by setting $E_d\cup \left\{E\right\}$ as the input edge set and $V_m$ as the input variable. Since the reconstructed graph produced in Phase 2 excludes not only the candidate edge $E$ but also the previously discovered spurious edges $E_d$, an additional edge is considered to be spurious if the reconstructed graph formed by $E_d\cup \left\{E\right\}$ further improves the BIC score in relation to the reconstructed graph formed by $E_d$ (i.e., returning a higher $\Delta$ than $\Delta_{MAX}$). If at least one edge satisfies this condition, SED identifies the edge that maximises $\Delta$ as spurious, and updates $G_{mod}$, $CSE$, $E_d$ and $\Delta_{MAX}$ accordingly. Phase 2 repeats the above process until no further edges in $CSE\left(V_m\right)$ can be identified as spurious, at which point $CSE\left(V_m\right)$ is removed from $CSE$. The SED algorithm then reverts to the Phase 1 to explore the next candidate spurious edge and its corresponding noisy variable in $CSE$.

Table{~\ref{tab: an example trace of the SED algorithm}} presents a series of figures that illustrate the outputs generated at various steps of the SED algorithm, by applying SED on the graph shown in Figure{~\ref{fig: original learned graph}}, which represents the Asia network learned by the HC algorithm from a synthetic data set with 10,000 samples and 5\% measurement error on each observed variable. The figures on the leftmost column represent the original graph and modified graphs after each step of SED, whereas the other figures illustrate the reconstructed process for each of the specified steps in SED. The text under each reconstructed graph indicates the input parameters, and the graphs highlighted in red represent the optimal reconstructed graphs selected in every step of SED to identify and eliminate spurious edges.
\begin{algorithm}[H]
\caption{Spurious Edge Detection (SED) algorithm}
\label{alg: SED algorithm}
\begin{algorithmic}[1]
\Procedure{SED}{$G, D$}

\underline{Input}: learned graph $G$, data set $D$

\underline{Output}: modified graph $G_{mod}$
\State $G_{mod} = G$
\State initialise $CSE$ from $G$
\Repeat
\LineComment{\% Phase 1: search spurious edge in $CSE$}
\If{$\underset{V\in CSE, E\in CSE(V)}{\text{max}}Reconstruction\left(G, V, \left\{E\right\}, D\right) > 0$}
\State $V_m, E_m = \underset{V\in CSE, E\in CSE(V)}{\text{argmax}}Reconstruction\left(G, V, \left\{E\right\}, D\right)$
\State $G_{mod} = G_{mod}\backslash E_m$
\State update $CSE$ given $G_{mod}$
\EndIf
\LineComment{\% Phase 2: search further spurious edges in $CSE\left(V_m\right)$}
\If{$CSE\left(V_m\right)\neq\emptyset$}
\State $E_d = \left\{E_m\right\}$
\State $\Delta_{MAX} = Reconstruction\left(G, V_m, \left\{E_m\right\}, D\right)$
\While{$\underset{E \in CSE(V_m)}{\text{max}}Reconstruction\left(G, V_m, E_d\cup\left\{E\right\}, D\right) > \Delta$}
\State $E_m = \underset{E \in CSE(V_m)}{\text{argmax}}Reconstruction\left(G, V_m, E_d\cup\left\{E\right\}, D\right)$
\State $G_{mod} = G_{mod}\backslash E_m$
\State update $CSE$ given $G_{mod}$
\State $E_d = E_d\cup \left\{E_m\right\}$
\State $\Delta_{MAX} = Reconstruction\left(G, V_m, E_d, D\right)$
\EndWhile
\State remove $CSE\left(V_m\right)$ from $CSE$
\EndIf
\Until{$G_{mod}$ is unchanged}
\Return $G_{mod}$
\EndProcedure
\end{algorithmic}
\end{algorithm}
The process starts at Figure~\ref{fig: original learned graph} which represents the original learned graph; i.e., the output of a structure learning algorithm. From this, SED obtains the candidate spurious edge set $CSE$ and enters Phase 1 where it executes the reconstruction procedure on each candidate edge, and removes the edge found to produce the highest $\Delta$ from the reconstruction procedure ($tub\rightarrow xray$ in this example). It then modifies the graph as shown in Figure{~{\ref{fig: modified graph after 1st round Phase 1}}}, from the graph highlighted in red in the first step. Note that Phase 1 also determined variable $either$ as a noisy variable on the basis that the eliminated edge $tub\rightarrow xray$ is explained by measurement error in $either$.

After the variable $either$ is identified as noisy, SED enters Phase 2 to investigate additional edges that might be spurious due to the noise present in this variable. In this example, there are further three edges in $CSE\left(either\right): \left\{smoke\rightarrow lung, lung\rightarrow dysp, xray\rightarrow dysp\right\}$ that could be explained by noisy $either$. SED examines each of these suspected spurious edges in $CSE\left(either\right)$, which removes one-by-one if they are found to further increase $\Delta$ compared with $\Delta_{MAX}$, where $\Delta_{MAX}$ is initialised by the highest $\Delta$ in the current round of Phase 1 and then updated by the highest $\Delta$ in the previous iteration of Phase 2. After two iterations in Phase 2, two spurious edges $lung\rightarrow dysp$ and $xray\rightarrow dysp$ can be detected from $CSE\left(either\right)$ since we are able to discover a better fitting model that attributes the presence of these two edges to the measurement error on the observation of $either$, thereby we get Figures{~\ref{fig: modified graph after 1st iteration of 1st round Phase 2}} and{~\ref{fig: modified graph after 2nd iteration of 1st round Phase 2}}. Note that, the input edge set of the reconstruction procedure in Phase 2 always contains the discovered spurious edges with the same corresponding noisy variable. Therefore, in the first iteration of Phase 2, the input edge set always contains the edge $tub\rightarrow xray$, and in the second iteration of Phase 2, the input edge set always contains edges $tub\rightarrow xray$ and $lung\rightarrow dysp$. Since no more edges in $CSE\left(either\right)$ can be detected as spurious in the third iteration of Phase 2, SED returns to Phase 1 to search spurious edges on the remaining $CSE$ set and the last spurious edge $smoke\rightarrow either$ could be identified at this time. As a result, SED updates the modified graph as shown in Figure~\ref{fig: modified graph after 2nd round Phase 1} which is also the final output of SED since there are no 3-vertex clique in it.
\setlength\LTleft{-1.5cm}
\begin{longtable}{>{\centering\arraybackslash}m{4.2cm}>{\centering\arraybackslash}m{3.7cm}>{\centering\arraybackslash}m{3.7cm}>{\centering\arraybackslash}m{3.7cm}>{\centering\arraybackslash}m{1cm}}
\captionsetup{format = hang}
        \toprule[1pt]
        \multirow{2}{*}[-1.5mm]{\parbox{4.5cm}{
            \centering
            \adjustbox{max totalheight = 4.5cm}{\begin{tikzpicture}[baseline = 0, scale = 0.85]
                \node[latent, minimum size = 1.2cm] (A) at (0, 0) {asia};
                \node[latent, minimum size = 1.2cm] (T) at (0, -2) {tub};
                \node[latent, minimum size = 1.2cm] (E) at (1, -4) {either};
                \node[latent, minimum size = 1.2cm] (L) at (2.5, -2) {lung};
                \node[latent, minimum size = 1.2cm] (S) at (4, 0) {smoke};
                \node[latent, minimum size = 1.2cm] (B) at (4, -3) {bronc};
                \node[latent, minimum size = 1.2cm] (D) at (4, -6) {dysp};
                \node[latent, minimum size = 1.2cm] (X) at (1, -6) {xray};
                \draw[->] (A) edge (T);
                \draw[->] (T) edge [bend right] (X) edge (E);
                \draw[->] (S) edge (L) edge (B) edge [bend right] (E);
                \draw[->] (B) edge (D);
                \draw[->] (L) edge (E) edge (D);
                \draw[->] (E) edge (D) edge (X);
                \draw[->] (X) edge (D);
                \end{tikzpicture}}
                \figcaption{original learned (input) graph}
                \label{fig: original learned graph}}} & \multicolumn{4}{c}{Reconstructed graphs formed in $1^{\textrm{st}}$ round Phase 1}\\
        \cmidrule(lr){2-5}
         & \shortstack{\adjustbox{max totalheight = 4.5cm}{\begin{tikzpicture}[baseline=0, scale = 0.85]
            \node[latent, minimum size = 1.2cm] (A) at (0, 0) {asia};
            \node[latent, minimum size = 1.2cm, dashed] (T) at (0, -2) {tub};
            \node[latent, minimum size = 1.2cm] (To) at (2, 0) {$\textrm{tub}^\textrm{o}$};
            \node[latent, minimum size = 1.2cm] (E) at (1, -4) {either};
            \node[latent, minimum size = 1.2cm] (L) at (2.5, -2) {lung};
            \node[latent, minimum size = 1.2cm] (S) at (4, 0) {smoke};
            \node[latent, minimum size = 1.2cm] (B) at (4, -3) {bronc};
            \node[latent, minimum size = 1.2cm] (D) at (4, -6) {dysp};
            \node[latent, minimum size = 1.2cm] (X) at (1, -6) {xray};
            \draw[->] (A) edge (T);
            \draw[->] (T) edge [bend right] (X) edge (To);
            \draw[->] (S) edge (L) edge (B);
            \draw[->] (S) edge [bend right] (E);
            \draw[->] (B) edge (D);
            \draw[->] (T) edge (E);
            \draw[->] (L) edge (E) edge (D);
            \draw[->] (E) edge (D);
            \draw[->] (X) edge (D);
        \end{tikzpicture}}\\ $V = tub$\\ $E = either\rightarrow xray$} & \tikzmarkin{a}(0.1, 0.67)(-0.1, 5.29)\shortstack{\adjustbox{max totalheight = 4.5cm}{\begin{tikzpicture}[baseline=0, scale = 0.85]
            \draw node[latent, minimum size = 1.2cm] (A) at (0, 0) {asia};
            \draw node[latent, minimum size = 1.2cm] (T) at (0, -2) {tub};
            \draw node[latent, minimum size = 1.2cm, dashed] (E) at (1, -4) {either};
            \draw node[latent, minimum size = 1.2cm] (Eo) at (-1, -6) {$\textrm{either}^\textrm{o}$};
            \draw node[latent, minimum size = 1.2cm] (L) at (2.5, -2) {lung};
            \draw node[latent, minimum size = 1.2cm] (S) at (4, 0) {smoke};
            \draw node[latent, minimum size = 1.2cm] (B) at (4, -3) {bronc};
            \draw node[latent, minimum size = 1.2cm] (D) at (4, -6) {dysp};
            \draw node[latent, minimum size = 1.2cm] (X) at (1, -6) {xray};
            \draw[->] (A) edge (T);
            \draw[->] (T) edge (E);
            \draw[->] (S) edge (L) edge (B) edge [bend right] (E);
            \draw[->] (B) edge (D);
            \draw[->] (L) edge (E) edge (D);
            \draw[->] (E) edge (X) edge (D) edge (Eo);
            \draw[->] (X) edge (D);
     \end{tikzpicture}}\\ $V = either$\\ $E = tub\rightarrow xray$}\tikzmarkend{a} & \shortstack{\adjustbox{max totalheight = 4.5cm}{\begin{tikzpicture}[baseline=0, scale = 0.85]
            \draw node[latent, minimum size = 1.2cm] (A) at (0, 0) {asia};
            \draw node[latent, minimum size = 1.2cm] (T) at (0, -2) {tub};
            \draw node[latent, minimum size = 1.2cm, dashed] (E) at (1, -4) {either};
            \draw node[latent, minimum size = 1.2cm] (Eo) at (-1, -6) {$\textrm{either}^\textrm{o}$};
            \draw node[latent, minimum size = 1.2cm] (L) at (2.5, -2) {lung};
            \draw node[latent, minimum size = 1.2cm] (S) at (4, 0) {smoke};
            \draw node[latent, minimum size = 1.2cm] (B) at (4, -3) {bronc};
            \draw node[latent, minimum size = 1.2cm] (D) at (4, -6) {dysp};
            \draw node[latent, minimum size = 1.2cm] (X) at (1, -6) {xray};
            \draw[->] (A) edge (T);
            \draw[->] (T) edge (E) edge [bend right] (X);
            \draw[->] (S) edge (B) edge [bend right] (E);
            \draw[->] (B) edge (D);
            \draw[->] (L) edge (E) edge (D);
            \draw[->] (E) edge (X) edge (D) edge (Eo);
            \draw[->] (X) edge (D);
     \end{tikzpicture}}\\ $V = either$\\ $E = smoke\rightarrow lung$} & \multirow{1}{*}[0.5cm]{$\cdots\cdots$}\\
     \cmidrule(lr){1-5}
     \multirow{2}{*}[-4mm]{\parbox{4.5cm}{
        \centering
        \adjustbox{max totalheight = 4.5cm}{\begin{tikzpicture}[baseline = 0, scale = 0.85]
            \node[latent, minimum size = 1.2cm] (A) at (0, 0) {asia};
            \node[latent, minimum size = 1.2cm] (T) at (0, -2) {tub};
            \node[latent, minimum size = 1.2cm] (E) at (1, -4) {either};
            \node[latent, minimum size = 1.2cm] (L) at (2.5, -2) {lung};
            \node[latent, minimum size = 1.2cm] (S) at (4, 0) {smoke};
            \node[latent, minimum size = 1.2cm] (B) at (4, -3) {bronc};
            \node[latent, minimum size = 1.2cm] (D) at (4, -6) {dysp};
            \node[latent, minimum size = 1.2cm] (X) at (1, -6) {xray};
            \draw[->] (A) edge (T);
            \draw[->] (T) edge (E);
            \draw[->] (S) edge (L) edge (B) edge [bend right] (E);
            \draw[->] (B) edge (D);
            \draw[->] (L) edge (E) edge (D);
            \draw[->] (E) edge (D) edge (X);
            \draw[->] (X) edge (D);
            \end{tikzpicture}}
        \figcaption{modified graph after $1^\textrm{st}$ round Phase 1}
        \label{fig: modified graph after 1st round Phase 1}}} &
    \multicolumn{4}{c}{Reconstructed graphs formed in $1^{\textrm{st}}$ iteration of $1^{\textrm{st}}$ round Phase 2}\\
     \cmidrule(lr){2-5}
      & \shortstack{\adjustbox{max totalheight = 4.5cm}{\begin{tikzpicture}[baseline=0, scale = 0.85]
            \node[latent, minimum size = 1.2cm] (A) at (0, 0) {asia};
            \node[latent, minimum size = 1.2cm] (T) at (0, -2) {tub};
            \node[latent, minimum size = 1.2cm, dashed] (E) at (1, -4) {either};
            \node[latent, minimum size = 1.2cm] (Eo) at (-1, -6) {$\textrm{either}^\textrm{o}$};
            \node[latent, minimum size = 1.2cm] (L) at (2.5, -2) {lung};
            \node[latent, minimum size = 1.2cm] (S) at (4, 0) {smoke};
            \node[latent, minimum size = 1.2cm] (B) at (4, -3) {bronc};
            \node[latent, minimum size = 1.2cm] (D) at (4, -6) {dysp};
            \node[latent, minimum size = 1.2cm] (X) at (1, -6) {xray};
            \draw[->] (A) edge (T);
            \draw[->] (T) edge (E);
            \draw[->] (S) edge (B) edge [bend right] (E);
            \draw[->] (B) edge (D);
            \draw[->] (L) edge (E) edge (D);
            \draw[->] (E) edge (X) edge (D) edge (Eo);
            \draw[->] (X) edge (D);
        \end{tikzpicture}}\\ $V_m = either$\\ $E = smoke\rightarrow lung$\\ $E_d = \left\{tub\rightarrow xray\right\}$} & \tikzmarkin{b}(0.1, 1.2)(-0.1, 5.82)\shortstack{\adjustbox{max totalheight = 4.5cm}{\begin{tikzpicture}[baseline=0, scale = 0.85]
            \draw node[latent, minimum size = 1.2cm] (A) at (0, 0) {asia};
            \draw node[latent, minimum size = 1.2cm] (T) at (0, -2) {tub};
            \draw node[latent, minimum size = 1.2cm, dashed] (E) at (1, -4) {either};
            \draw node[latent, minimum size = 1.2cm] (Eo) at (-1, -6) {$\textrm{either}^\textrm{o}$};
            \draw node[latent, minimum size = 1.2cm] (L) at (2.5, -2) {lung};
            \draw node[latent, minimum size = 1.2cm] (S) at (4, 0) {smoke};
            \draw node[latent, minimum size = 1.2cm] (B) at (4, -3) {bronc};
            \draw node[latent, minimum size = 1.2cm] (D) at (4, -6) {dysp};
            \draw node[latent, minimum size = 1.2cm] (X) at (1, -6) {xray};
            \draw[->] (A) edge (T);
            \draw[->] (T) edge (E);
            \draw[->] (S) edge (B) edge [bend right] (E) edge (L);
            \draw[->] (B) edge (D);
            \draw[->] (L) edge (E);
            \draw[->] (E) edge (X) edge (D) edge (Eo);
            \draw[->] (X) edge (D);
        \end{tikzpicture}}\\ $V_m = either$\\ $E = lung\rightarrow dysp$\\ $E_d = \left\{tub\rightarrow xray\right\}$}\tikzmarkend{b} & \shortstack{\adjustbox{max totalheight = 4.5cm}{\begin{tikzpicture}[baseline=0, scale = 0.85]
            \draw node[latent, minimum size = 1.2cm] (A) at (0, 0) {asia};
            \draw node[latent, minimum size = 1.2cm] (T) at (0, -2) {tub};
            \draw node[latent, minimum size = 1.2cm, dashed] (E) at (1, -4) {either};
            \draw node[latent, minimum size = 1.2cm] (Eo) at (-1, -6) {$\textrm{either}^\textrm{o}$};
            \draw node[latent, minimum size = 1.2cm] (L) at (2.5, -2) {lung};
            \draw node[latent, minimum size = 1.2cm] (S) at (4, 0) {smoke};
            \draw node[latent, minimum size = 1.2cm] (B) at (4, -3) {bronc};
            \draw node[latent, minimum size = 1.2cm] (D) at (4, -6) {dysp};
            \draw node[latent, minimum size = 1.2cm] (X) at (1, -6) {xray};
            \draw[->] (A) edge (T);
            \draw[->] (T) edge (E);
            \draw[->] (S) edge (B) edge [bend right] (E) edge (L);
            \draw[->] (B) edge (D);
            \draw[->] (L) edge (E) edge (D);
            \draw[->] (E) edge (X) edge (D) edge (Eo);
        \end{tikzpicture}}\\ $V_m = either$\\ $E = xray\rightarrow dysp$\\ $E_d = \left\{tub\rightarrow xray\right\}$}\\
        \cmidrule(lr){1-5}
        \newpage
        \cmidrule(lr){1-5}
        \multirow{2}{*}[-1mm]{\parbox{4.5cm}{
        \centering
        \adjustbox{max totalheight = 4.5cm}{\begin{tikzpicture}[baseline = 0, scale = 0.85]
            \node[latent, minimum size = 1.2cm] (A) at (0, 0) {asia};
            \node[latent, minimum size = 1.2cm] (T) at (0, -2) {tub};
            \node[latent, minimum size = 1.2cm] (E) at (1, -4) {either};
            \node[latent, minimum size = 1.2cm] (L) at (2.5, -2) {lung};
            \node[latent, minimum size = 1.2cm] (S) at (4, 0) {smoke};
            \node[latent, minimum size = 1.2cm] (B) at (4, -3) {bronc};
            \node[latent, minimum size = 1.2cm] (D) at (4, -6) {dysp};
            \node[latent, minimum size = 1.2cm] (X) at (1, -6) {xray};
            \draw[->] (A) edge (T);
            \draw[->] (T) edge (E);
            \draw[->] (S) edge (L) edge (B) edge [bend right] (E);
            \draw[->] (B) edge (D);
            \draw[->] (L) edge (E);
            \draw[->] (E) edge (D) edge (X);
            \draw[->] (X) edge (D);
            \end{tikzpicture}}
        \figcaption{modified graph after $1^\textrm{st}$ iteration of $1^\textrm{st}$ round Phase 2}
        \label{fig: modified graph after 1st iteration of 1st round Phase 2}}} & \multicolumn{4}{c}{Reconstructed graphs formed in $2^{\textrm{nd}}$ iteration of $1^{\textrm{st}}$ round Phase 2}\\
        \cmidrule(lr){2-5}
        &
        \multicolumn{1}{M{3cm}}{\shortstack{\adjustbox{max totalheight = 4.5cm}{\begin{tikzpicture}[baseline=0, scale = 0.85]
            \draw node[latent, minimum size = 1.2cm] (A) at (0, 0) {asia};
            \draw node[latent, minimum size = 1.2cm] (T) at (0, -2) {tub};
            \draw node[latent, minimum size = 1.2cm, dashed] (E) at (1, -4) {either};
            \draw node[latent, minimum size = 1.2cm] (Eo) at (-1, -6) {$\textrm{either}^\textrm{o}$};
            \draw node[latent, minimum size = 1.2cm] (L) at (2.5, -2) {lung};
            \draw node[latent, minimum size = 1.2cm] (S) at (4, 0) {smoke};
            \draw node[latent, minimum size = 1.2cm] (B) at (4, -3) {bronc};
            \draw node[latent, minimum size = 1.2cm] (D) at (4, -6) {dysp};
            \draw node[latent, minimum size = 1.2cm] (X) at (1, -6) {xray};
            \draw[->] (A) edge (T);
            \draw[->] (T) edge (E);
            \draw[->] (S) edge (B) edge [bend right] (E);
            \draw[->] (B) edge (D);
            \draw[->] (L) edge (E);
            \draw[->] (E) edge (X) edge (D) edge (Eo);
            \draw[->] (X) edge (D);
        \end{tikzpicture}}\\ $V_m = either$\\ $E = smoke\rightarrow lung$\\ $E_d = \left\{tub\rightarrow xray, lung\rightarrow dysp\right\}$}} & \multicolumn{3}{M{7cm}}{\tikzmarkin{c}(-0.75, 1.2)(0.75, 5.82)\shortstack{\adjustbox{max totalheight = 4.5cm}{\begin{tikzpicture}[baseline=0, scale = 0.85]
            \draw node[latent, minimum size = 1.2cm] (A) at (0, 0) {asia};
            \draw node[latent, minimum size = 1.2cm] (T) at (0, -2) {tub};
            \draw node[latent, minimum size = 1.2cm, dashed] (E) at (1, -4) {either};
            \draw node[latent, minimum size = 1.2cm] (Eo) at (-1, -6) {$\textrm{either}^\textrm{o}$};
            \draw node[latent, minimum size = 1.2cm] (L) at (2.5, -2) {lung};
            \draw node[latent, minimum size = 1.2cm] (S) at (4, 0) {smoke};
            \draw node[latent, minimum size = 1.2cm] (B) at (4, -3) {bronc};
            \draw node[latent, minimum size = 1.2cm] (D) at (4, -6) {dysp};
            \draw node[latent, minimum size = 1.2cm] (X) at (1, -6) {xray};
            \draw[->] (A) edge (T);
            \draw[->] (T) edge (E);
            \draw[->] (S) edge (B) edge (L) edge [bend right] (E);
            \draw[->] (B) edge (D);
            \draw[->] (L) edge (E);
            \draw[->] (E) edge (X) edge (D) edge (Eo);
        \end{tikzpicture}}\\ $V_m = either$\\ $E = xray\rightarrow dysp$\\ $E_d = \left\{tub\rightarrow xray, lung\rightarrow dysp\right\}$}\tikzmarkend{c}}\\
        \cmidrule(lr){1-5}
        \multirow{2}{*}[-2mm]{\parbox{4.5cm}{
        \centering
        \adjustbox{max totalheight = 4.5cm}{\begin{tikzpicture}[baseline = 0, scale = 0.85]
            \node[latent, minimum size = 1.2cm] (A) at (0, 0) {asia};
            \node[latent, minimum size = 1.2cm] (T) at (0, -2) {tub};
            \node[latent, minimum size = 1.2cm] (E) at (1, -4) {either};
            \node[latent, minimum size = 1.2cm] (L) at (2.5, -2) {lung};
            \node[latent, minimum size = 1.2cm] (S) at (4, 0) {smoke};
            \node[latent, minimum size = 1.2cm] (B) at (4, -3) {bronc};
            \node[latent, minimum size = 1.2cm] (D) at (4, -6) {dysp};
            \node[latent, minimum size = 1.2cm] (X) at (1, -6) {xray};
            \draw[->] (A) edge (T);
            \draw[->] (T) edge (E);
            \draw[->] (S) edge (L) edge (B) edge [bend right] (E);
            \draw[->] (B) edge (D);
            \draw[->] (L) edge (E);
            \draw[->] (E) edge (D) edge (X);
            \end{tikzpicture}}
        \figcaption{modified graph after $2^\textrm{nd}$ iteration of $1^\textrm{st}$ round Phase 2}
        \label{fig: modified graph after 2nd iteration of 1st round Phase 2}}} & \multicolumn{4}{c}{Reconstructed graph formed in $3^\textrm{rd}$ iteration of $1^\textrm{st}$ round Phase 2}\\
        \cmidrule(lr){2-5} &
        \multicolumn{2}{M{5cm}}{\shortstack{\adjustbox{max totalheight = 4.5cm}{\begin{tikzpicture}[baseline=0, scale = 0.85]
            \draw node[latent, minimum size = 1.2cm] (A) at (0, 0) {asia};
            \draw node[latent, minimum size = 1.2cm] (T) at (0, -2) {tub};
            \draw node[latent, minimum size = 1.2cm, dashed] (E) at (1, -4) {either};
            \draw node[latent, minimum size = 1.2cm] (Eo) at (-1, -6) {$\textrm{either}^\textrm{o}$}; 
            \draw node[latent, minimum size = 1.2cm] (L) at (2.5, -2) {lung};
            \draw node[latent, minimum size = 1.2cm] (S) at (4, 0) {smoke};
            \draw node[latent, minimum size = 1.2cm] (B) at (4, -3) {bronc};
            \draw node[latent, minimum size = 1.2cm] (D) at (4, -6) {dysp};
            \draw node[latent, minimum size = 1.2cm] (X) at (1, -6) {xray};
            \draw[->] (A) edge (T);
            \draw[->] (T) edge (E);
            \draw[->] (S) edge (B) edge [bend right] (E);
            \draw[->] (B) edge (D);
            \draw[->] (L) edge (E);
            \draw[->] (E) edge (X) edge (D) edge (Eo);
        \end{tikzpicture}}\\ $V_m = either$\\ $E = smoke\rightarrow lung$\\ $E_d = \left\{tub\rightarrow xray, lung\rightarrow dysp, xray\rightarrow dysp\right\}$}} & &\\
        \cmidrule(lr){1-5}
        \multirow{2}{*}{\parbox{4.5cm}{
        \centering
        \adjustbox{max totalheight = 4.5cm}{\begin{tikzpicture}[baseline = 0, scale = 0.85]
            \node[latent, minimum size = 1.2cm] (A) at (0, 0) {asia};
            \node[latent, minimum size = 1.2cm] (T) at (0, -2) {tub};
            \node[latent, minimum size = 1.2cm] (E) at (1, -4) {either};
            \node[latent, minimum size = 1.2cm] (L) at (2.5, -2) {lung};
            \node[latent, minimum size = 1.2cm] (S) at (4, 0) {smoke};
            \node[latent, minimum size = 1.2cm] (B) at (4, -3) {bronc};
            \node[latent, minimum size = 1.2cm] (D) at (4, -6) {dysp};
            \node[latent, minimum size = 1.2cm] (X) at (1, -6) {xray};
            \draw[->] (A) edge (T);
            \draw[->] (T) edge (E);
            \draw[->] (S) edge (L) edge (B) edge [bend right] (E);
            \draw[->] (B) edge (D);
            \draw[->] (L) edge (E);
            \draw[->] (E) edge (D) edge (X);
            \end{tikzpicture}}
        \figcaption{modified graph after $3^\textrm{rd}$ iteration of $1^\textrm{st}$ round Phase 2}
        \label{fig: modified graph after 3rd iteration of 1st round Phase 2}}} & \multicolumn{4}{c}{Reconstructed graphs formed in $2^\textrm{nd}$ round Phase 1}\\
        \cmidrule(lr){2-5} & \multicolumn{1}{M{2cm}}{\shortstack{\adjustbox{max totalheight = 4.5cm}{\begin{tikzpicture}[baseline=0, scale = 0.85]
            \draw node[latent, minimum size = 1.2cm] (A) at (0, 0) {asia};
            \draw node[latent, minimum size = 1.2cm] (T) at (0, -2) {tub};
            \draw node[latent, minimum size = 1.2cm] (E) at (1, -4) {either};
            \draw node[latent, minimum size = 1.2cm] (L) at (2.5, -2) {lung};
            \draw node[latent, minimum size = 1.2cm, dashed] (S) at (4, 0) {smoke};
            \draw node[latent, minimum size = 1.2cm] (So) at (2, 0) {$\textrm{smoke}^\textrm{o}$};
            \draw node[latent, minimum size = 1.2cm] (B) at (4, -3) {bronc};
            \draw node[latent, minimum size = 1.2cm] (D) at (4, -6) {dysp};
            \draw node[latent, minimum size = 1.2cm] (X) at (1, -6) {xray};
            \draw[->] (A) edge (T);
            \draw[->] (T) edge (E) edge [bend right] (X);
            \draw[->] (S) edge (L) edge (B) edge [bend right] (E) edge (So);
            \draw[->] (B) edge (D);
            \draw[->] (L) edge (D);
            \draw[->] (E) edge (X) edge (D);
            \draw[->] (X) edge (D);
        \end{tikzpicture}}\\ $V = smoke$\\ $E = lung\rightarrow either$}} & \multicolumn{1}{M{1cm}}{\tikzmarkin{d}(0.1, 0.75)(-0.1, 5.37)\shortstack{\adjustbox{max totalheight = 4.5cm}{\begin{tikzpicture}[baseline=0, scale = 0.85]
            \draw node[latent, minimum size = 1.2cm] (A) at (0, 0) {asia};
            \draw node[latent, minimum size = 1.2cm] (T) at (0, -2) {tub};
            \draw node[latent, minimum size = 1.2cm] (E) at (1, -4) {either};
            \draw node[latent, minimum size = 1.2cm, dashed] (L) at (2.5, -2) {lung};
            \draw node[latent, minimum size = 1.2cm] (S) at (4, 0) {smoke};
            \draw node[latent, minimum size = 1.2cm] (Lo) at (2, 0) {$\textrm{lung}^\textrm{o}$};
            \draw node[latent, minimum size = 1.2cm] (B) at (4, -3) {bronc};
            \draw node[latent, minimum size = 1.2cm] (D) at (4, -6) {dysp};
            \draw node[latent, minimum size = 1.2cm] (X) at (1, -6) {xray};
            \draw[->] (A) edge (T);
            \draw[->] (T) edge (E) edge [bend right] (X);
            \draw[->] (S) edge (L) edge (B);
            \draw[->] (B) edge (D);
            \draw[->] (L) edge (D) edge (E) edge (Lo);
            \draw[->] (E) edge (X) edge (D);
            \draw[->] (X) edge (D);
        \end{tikzpicture}}\\ $V = lung$\\ $E = smoke\rightarrow either$}\tikzmarkend{d}
        } & &\\
        \cmidrule(lr){1-5}
        \newpage
        \cmidrule(lr){1-5}
        \parbox{4.5cm}{
        \centering
        \adjustbox{max totalheight = 4.5cm}{\begin{tikzpicture}[baseline = 0, scale = 0.85]
            \node[latent, minimum size = 1.2cm] (A) at (0, 0) {asia};
            \node[latent, minimum size = 1.2cm] (T) at (0, -2) {tub};
            \node[latent, minimum size = 1.2cm] (E) at (1, -4) {either};
            \node[latent, minimum size = 1.2cm] (L) at (2.5, -2) {lung};
            \node[latent, minimum size = 1.2cm] (S) at (4, 0) {smoke};
            \node[latent, minimum size = 1.2cm] (B) at (4, -3) {bronc};
            \node[latent, minimum size = 1.2cm] (D) at (4, -6) {dysp};
            \node[latent, minimum size = 1.2cm] (X) at (1, -6) {xray};
            \draw[->] (A) edge (T);
            \draw[->] (T) edge (E);
            \draw[->] (S) edge (L) edge (B);
            \draw[->] (B) edge (D);
            \draw[->] (L) edge (E);
            \draw[->] (E) edge (D) edge (X);
            \end{tikzpicture}}
        \figcaption{modified graph after $2^\textrm{nd}$ round Phase 1}
        \label{fig: modified graph after 2nd round Phase 1}} & & & &\\
        \bottomrule[1pt]
        \caption{The steps of the SED algorithm in modifying the Asia graph learned by HC (Figure~\ref{fig: original learned graph}) from synthetic data of sample size 10,000 with 5\% measurement error on all variables.}
        \label{tab: an example trace of the SED algorithm}
\end{longtable}
\section{Empirical evaluation}
\label{sec: empirical evaluation}
We validate the effectiveness of the SED algorithm, which can be viewed as a structure learning addon, by applying it to four well-established structure learning algorithms spanning different classes of learning. These are the score-based HC and ILP, the constraint-based PC-stable and the hybrid H2PC. We use the bnlearn R package~\citep{scutari2010learning} to test the effect on HC and H2PC, the rcausal R package~\citep{Chirayu2019} for PC-Stable and the pygobnilp python package~\citep{cussens2011bayesian} for ILP.

We use the BIC score as the objective function for the two score-based HC and ILP algorithms, including the score-based phase in H2PC. For the constraint-based algorithm PC-Stable, including the constraint-based phase in H2PC, we use the G-square test as the statistical test and set the threshold for rejecting the null hypothesis at 0.05. Lastly, ILP's maximum in-degree is set to 3 (default hyperparameter). Because BIC is a score-equivalent objective function, HC, ILP and H2PC produce a DAG from a Markov Equivalent Class of DAGs, and which we convert into the corresponding CPDAGs to be used as the input of the SED algorithm; i.e., input graph $G$ in Algorithm~\ref{alg: SED algorithm}. We employ two metrics to evaluate the learned CPDAGs. These are the F1 score which combines the $Precision$ and $Recall$ in the following form:
\begin{equation}
    F1 = 2\frac{Precision\cdot Recall}{Precision + Recall}
\end{equation}
and the Structural Hamming Distance (SHD)~\citep{tsamardinos2006max} which represents the number of edge additions, edge removals and arc reversals required to move from the learned graph to true graph.

The experiments are based on synthetic data generated from seven real-world BN models that are publicly available in the bnlearn repository~\citep{scutari2020}. These are the Asia, Alarm, Child, Insurance, Mildew, Water and Hailfinder networks. For each network, we generated seven error-free data sets with the sample sizes ranging from 100 to 100,000. Moreover, for each error-free data set we generated a noisy data set by assigning a randomised error rate $\alpha_i$, with an upper bound of 0.1, to every variable $V_i$ in a network. Specifically, for each state $v_i^l$ of an error-free variable $V_i$, we assign a randomised error rate $\alpha_i^l$, with an upper bound of $\alpha_i$, where the probability of the error for each state $v_i^l$ follows a Dirichlet distribution. This process produces the corresponding noisy conditional probability distribution of each observed variable $V_i^o$ based on the following equation:
\begin{equation}
    P\left(V_i^o\mid V_i = v_i^l\right) = \begin{cases}
      \alpha^l_{i1}, & V_i^o = v_i^1 \\
      \alpha^l_{i2}, & V_i^o = v_i^2 \\
      \quad\vdots & \quad\vdots\\
      1 - \alpha^l_i, & V_i^o = v_i^l \\
      \quad\vdots & \quad\vdots\\
      \alpha^l_{ir_i}, & V_i^o = v_i^{r_i}
    \end{cases}
    \label{equ: generation mechanism}
\end{equation}
where the parameters $\left(\alpha^1_{i1}, \alpha^l_{i2}, \cdots, \alpha^l_{ir_i}\right)\sim \alpha^l_i Dir\underbrace{\left(1, \ldots, 1\right)}_{r_i-1}$, $r_i$ represent the number of states in $V_i$.
\subsection{Results}
\label{subsec: results}
We explore the performance of the SED algorithm on both error-free and observed noisy data sets. Figure~\ref{fig: average_f1} presents the F1 scores produced by four algorithms averaged across all seven networks, on both the error-free and noisy data sets, with and without SED modifications. Note that for error-free data sets, there is no visible difference in the F1 scores between the original learned graphs and the modified graphs which implies that the SED algorithm has made largely insignificant modifications to the graphs learned from error-free data sets. On the other hand, the modifications made on graphs learned from noisy data have led to noticeable improvements, and particularly in cases where data have higher sample size. Specifically, the improvements on graphs produced by score-based HC and ILP are somewhat stronger compared to the improvements on graphs produced by PC-Stable and H2PC. This can be explained by the larger deterioration in the performance of score-based HC and ILP after adding noise to the data, larger negative repercussions of data noise increase the opportunities of SED to discover and correct errors. These results are consistent with the empirical analysis presented in Figure~\ref{fig: average number of cliques in learned graphs} which shows that score-based learning is more sensitive to measurement error compared to constraint-based learning.

Lastly, the observation that the SED modifications provide a greater benefit as the sample size increases, can be explained by the fact that the structure learning algorithms generally tend to produce more edges when the input data contains higher samples, such that more false positive 3-vertex cliques that could be detected and corrected by SED. We present the results of 3-vertex cliques in Appendix~\ref{app: results of 3-vertex cliques}. Another explanation is that the EM learning used by SED is less effective when the sample size of the input data is low.
\begin{figure}[H]
    \centering
    \includegraphics[width =0.85\linewidth]{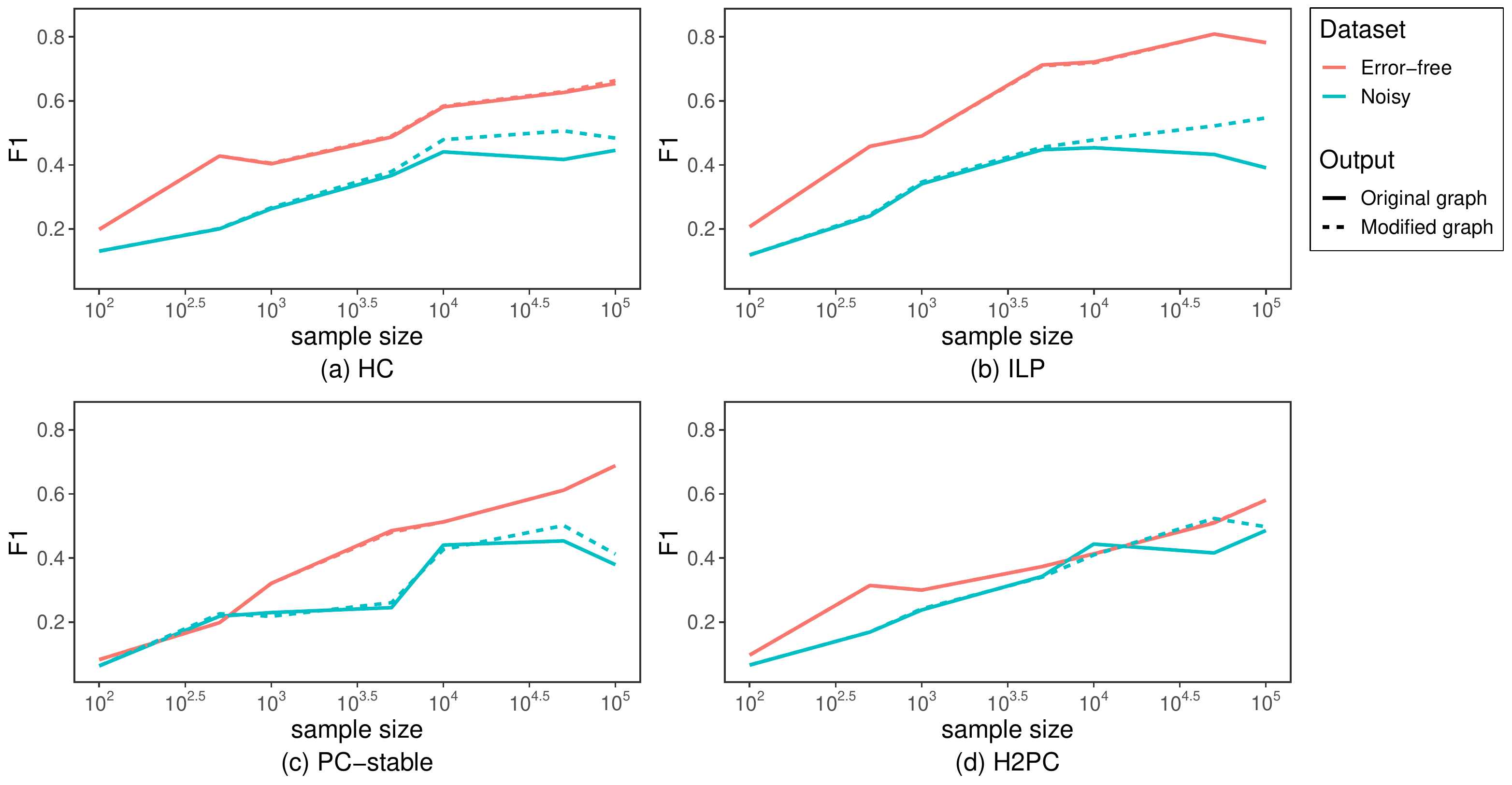}
    \caption{The average F1 scores produced by the four algorithms where solid lines represent the scores before SED modifications, dashed lines the scores following SED modifications, red lines the scores based on error-free data, and light blue lines the scores based on noisy data.}
    \label{fig: average_f1}
\end{figure}
\begin{figure}[H]
    \centering
    \includegraphics[width =0.85\linewidth]{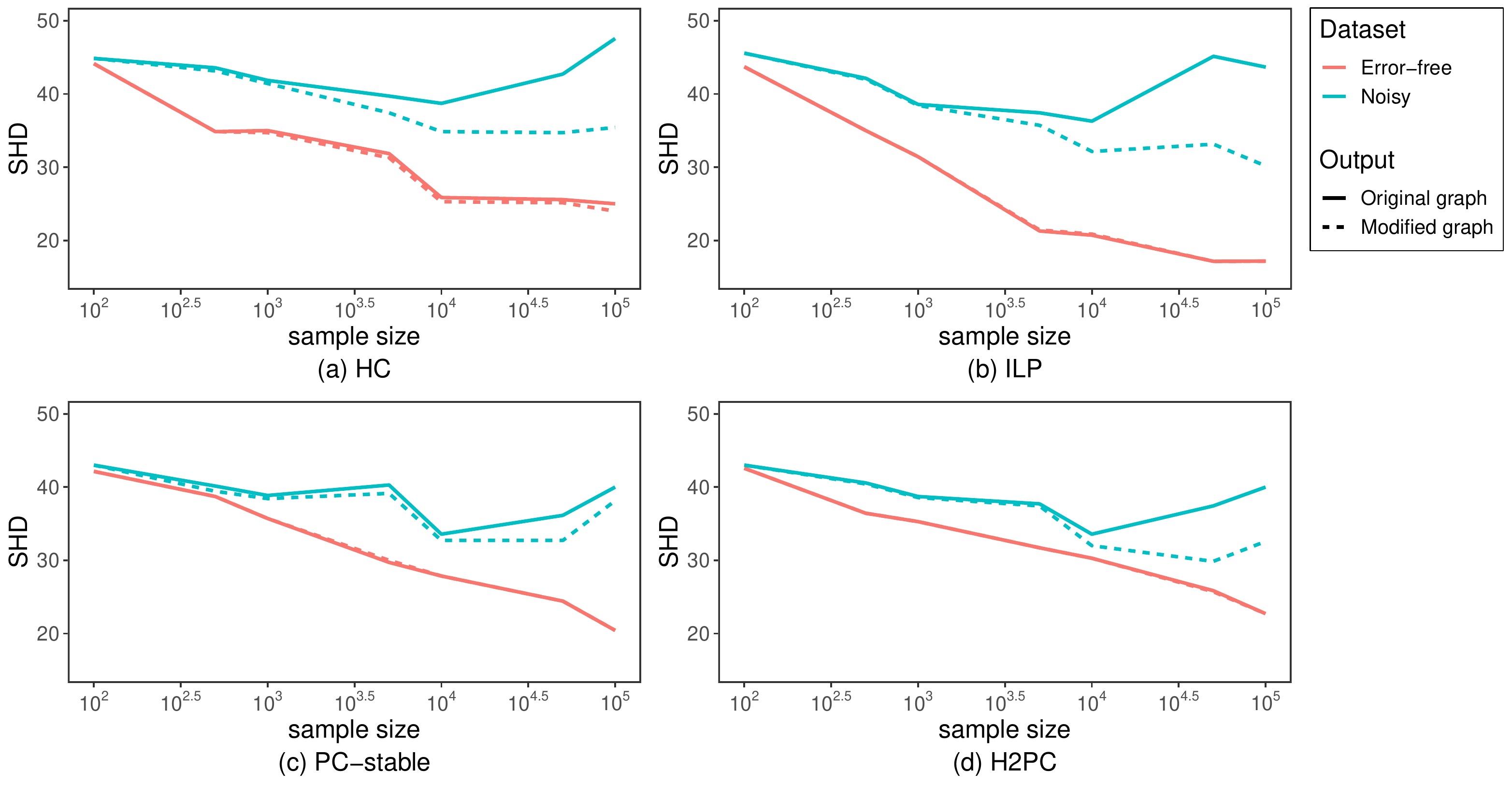}
    \caption{The average SHD scores produced by the four algorithms where solid lines represent the score before SED modifications, dashed lines the scores following SED modifications, red lines the scores based on error-free data, and light blue lines the scores based on noisy data.}
    \label{fig: average_shd}
\end{figure}
Figure~\ref{fig: average_shd} repeats these results for the SHD score. While the results are largely consistent with those based on the F1 score, the improvements appear to be major and more consistent in terms of SHD score, and also reveal minor improvements on graphs learned from error-free data. Overall, the SHD results suggest that the SED algorithm improves the graphs learned by the other algorithms by successfully eliminating a greater number of false positive, in relation to true positive, edges.

Table~\ref{tab: the summary} compares the scores between modified and original graphs. Interestingly, the results show that even when no measurement error exists in the input data (i.e., error-free cases), the SED modifications generally maintain or slightly improve the accuracy of the original graph. Overall, without measurement error in the data, the modifications increased the F1 score in 15 (out of 196) graphs and decreased it in seven graphs. Similarly, the modifications increased the SHD score in 15 graphs and decreased it in four graphs. However, major improvements are observed only in the graphs produced by HC, and this could be explained by the simplicity of HC which tends to stuck in a local optimum graph that may contain more false positive edges compared to the graphs produced by other algorithms, thereby giving more opportunities to the SED algorithm to correct the graph.

When measurement error exists in the input data (i.e., noisy cases), the SED modifications improve 92 out of the 196 (or 47\%) graphs according to the F1 score more, or 104 (53\%) according to the SHD score; although they also worsen the F1 score of 28 graphs (14\%), or six graphs (3\%) according to the SHD score. These percentages are generally consistent across all four algorithms irrespective of their class of learning.
\begin{table}[H]
  \centering
    \begin{tabular}{cccccc}
    \toprule[1pt]
    \multirow{3}[6]{*}{Algorithms} & \multirow{3}[4]{*}{\shortstack{Modified graph\\vs\\Original graph}} & \multicolumn{4}{c}{Evaluation metric} \\
\cmidrule(lr){3-6} & & \multicolumn{2}{c}{F1} & \multicolumn{2}{c}{SHD} \\
\cmidrule(lr){3-4}\cmidrule(lr){5-6} & & Error-free & Noisy & Error-free & Noisy\\
    \cmidrule(lr){1-6}
    \multirow{3}[2]{*}{HC} & Better & 13 (27\%) & 26 (53\%) & 14 (29\%) & 29 (59\%) \\
    & Same & 33 (67\%) & 18 (37\%) & 35 (71\%) & 19 (39\%)\\
    & Worse & 3 (6\%) & 5 (10\%) & 0 (0\%) & 1 (2\%) \\
    \cmidrule(lr){1-6}
    \multirow{3}[2]{*}{ILP} & Better & 0 (0\%) & 23 (47\%) & 0 (0\%) & 25 (51\%) \\
    & Same & 47 (96\%) & 21 (43\%) & 47 (96\%) & 24 (49\%) \\
    & Worse & 2 (4\%) & 5 (10\%) & 2 (4\%) & 0 (0\%) \\
    \cmidrule(lr){1-6}
    \multirow{3}[2]{*}{PC-stable} & Better & 0 (0\%) & 21 (43\%) & 0 (0\%) & 29 (59\%) \\
    & Same & 47 (96\%) & 16 (33\%) & 47 (96\%) & 18 (37\%) \\
    & Worse & 2 (4\%) & 12 (24\%) & 2 (4\%) & 2 (4\%) \\
    \cmidrule(lr){1-6}
    \multirow{3}[2]{*}{H2PC} & Better & 2 (4\%) & 22 (45\%) & 1 (2\%) & 21 (43\%) \\
    & Same & 47 (96\%) & 21 (43\%) & 48 (98\%) & 25 (51\%) \\
    & Worse & 0 (0\%) & 6 (12\%) & 0 (0\%) & 3 (6\%) \\
    \cmidrule(lr){1-6}
    \multirow{3}[2]{*}{Overall} & Better & 15 (8\%) & 92 (47\%) & 15 (8\%) & 104 (53\%) \\
    & Same & 174 (89\%) & 76 (39\%) & 177 (90\%) & 86 (44\%) \\
    & Worse & 7 (4\%) & 28 (14\%) & 4 (2\%) & 6 (3\%) \\
    \bottomrule[1pt]
    \end{tabular}
    \caption{Summary statistics on score difference between modified and original graphs, distributed per algorithm per evaluation metric per data set assumption.}
  \label{tab: the summary}
\end{table}
\section{Concluding remarks}
\label{sec: conclusion and discussions}
This paper described the SED algorithm that can be viewed as a structure learning addon which can be incorporated as an additional learning phase to discrete BN structure learning algorithms. The purpose of SED is to discover and eliminate potential false positive edges that structure learning algorithms tend to produce when learning graphs from data that contain measurement error, irrespective of their class of learning.

We have applied SED modifications to graphs produced by algorithms spanning different classes of learning (i.e., score-based, constraint-based and hybrid learning). The results are based on both error-free and noisy synthetic data that vary in sample size, and which have been generated from real-world BN models that also greatly vary in terms of the size of network. Overall, the results show that SED generally maintains, or slightly improves, the graphs produced by other algorithms when these graphs are learned from error-free data, and effectively improves the graphs learned from noisy-data.

A limitation of our work is that the proposed algorithm relies on the assumption that a noisy variable is independent of other variables in the network conditional on its error-free version, and this assumption is often considered to be too strong in some fields~\citep{hu2008identification}. For example, a survey on unemployment data by~\cite{bound2001measurement} shows that unemployment rate is underestimated, and the underestimation error appears to be dependent on the demographic characteristics of the respondent, such as age and sex. Moreover, since the problem of measurement error can be viewed as a special case of a hidden variable problem, future work could extend the application of this approach to structure learning algorithms designed to learn graphical structures under the assumptions of causal insufficiency~\citep{zhang2008completeness, ogarrio2016hybrid}.

\acks{This research was supported by the ERSRC Fellowship project EP/S001646/1 on Bayesian Artificial Intelligence for Decision Making under Uncertainty~\citep{constantinou2018bayesian}, and by The Alan Turing Institute in the UK under the EPSRC grant EP/N510129/1.}

\renewcommand{\theHsection}{A\arabic{section}}
\appendix
\section{EM algorithm and BIC score}
\label{app: EM algorithm adn BIC score}
The EM algorithm \citep{lauritzen1995algorithm} is an iterative process that computes the Maximum Likelihood Estimation (MLE) of the parameters $\theta$ for a given structure and from incomplete data. Generally, The EM algorithm can be decomposed in two steps, known as the Expectation step (E step) and the Maximization step (M step). In the E step, the EM algorithm computes the expected log-likelihood function $Q\left(\theta\mid\theta^t\right)$ based on $\theta^t$ obtained with each sample (data row) in the data. Assuming $\bm X$ represents a set of variables with missing values in data set $D$ with sample size $N$, the expectation of the LL function is:
\begin{equation}
Q\left(\theta\mid\theta^t\right) = \sum\limits^N_{m = 1}\sum\limits_{\bm x\in\Omega_{\bm X}}\mathbb{P}\left(\bm X = \bm x\mid D_m, \theta^t\right)\mathrm{log}\mathbb{P}\left(\bm X = \bm x, D_m\mid\theta\right)
\end{equation}
At the M step, the EM algorithm revises $\theta$ by maximising the expected LL:
\begin{equation}
    \theta^{t + 1} = \argmax\limits_\theta Q\left(\theta\mid\theta^t\right)
    \label{equ: general M step}
\end{equation}
The EM algorithm starts from a random initialisation of $\theta$ and terminates when the LL converges over a given threshold $\epsilon$:
\begin{equation}
\mathrm{log}\mathbb{P}\left(D\mid\theta^t\right) - \mathrm{log}\mathbb{P}\left(D\mid\theta^{t - 1}\right) < \epsilon
\label{equ: convergence}
\end{equation}
where $\epsilon$ is a threshold for judging whether the process is converged.

Applying EM learning on a BN requires that we compute:
\begin{equation}
    \widetilde{N}_{ijk}^t = \sum\limits_m\mathbb{P}\left(V_i = k, pa\left(V_i\right) = j\mid D_m, \theta^t\right)
\end{equation}
for the E step, where $\widetilde{N}_{ijk}$ represents the expected count of number of records where the value of variable $V_i = k$ and its parents $pa\left(V_i\right) = j$. For the M step, the solution of equation \ref{equ: general M step} has the following form:
\begin{equation}
    \theta^{t + 1} = \frac{\widetilde{N}_{ijk}^t}{\sum_k\widetilde{N}_{ijk}^t}
\end{equation}

Once the parameters of the model are estimated, the LL obtained by EM is used as the LL input in the BIC equation to measure the goodness-of-fit of a given reconstructed graph $G_r$ with respect to the observed data. Specifically, the BIC score of a BN model $M$ and corresponding data set $D$ is defined as:
\begin{equation}
    BIC\left(M\mid D\right) = \textrm{log}\mathbb{P}\left(D\mid M\right) - \frac{1}{2}\mathrm{log}\left(N\right)d
    \label{equ: BIC}
\end{equation}
where $N$ is the sample size of data set $D$ and $d = \sum\limits_i\left(r_i - 1\right)q_i$ is the number of free parameters in $M$, where $r_i$ represents the number of states in variable $V_i$ and $q_i$ represents the number of configuration of the parents of $V_i$. When computing the BIC score on Bayesian Network without hidden variables, due to the decomposability of the LL function, the equation \ref{equ: BIC} can be simplified as:
\begin{equation}
    BIC\left(M\mid D\right) = \sum\limits_{ijk}N_{ijk}\frac{N_{ijk}}{N_{ij}} - \frac{1}{2}\mathrm{log}\left(N\right)d
    \label{equ: simplified BIC}
\end{equation}
where $N_{ijk}$ is the number of counts when $V_i = k$ and $pa\left(V_i\right) = k$ and $N_{ij} = \sum_k N_{ijk}$.

However, when computing the BIC score on a reconstructed graph, the LL function is not decomposable due to the presence of hidden variable, which means that we cannot use equation \ref{equ: simplified BIC}. Instead, we use the LL converged at the final M-step as described in equation \ref{equ: convergence}. Moreover, if the learned graph or a reconstructed graph is a CPDAG, we will randomly select a DAG from the Markov equivalence class of the CPDAG and retrieve the BIC score of that DAG to represent the BIC score of the CPDAG, since the BIC score is equivalent for Markov equivalent DAGs.
\section{Results of 3-vertex cliques}
\label{app: results of 3-vertex cliques}
\begin{figure}[H]
    \centering
    \includegraphics[width=\linewidth]{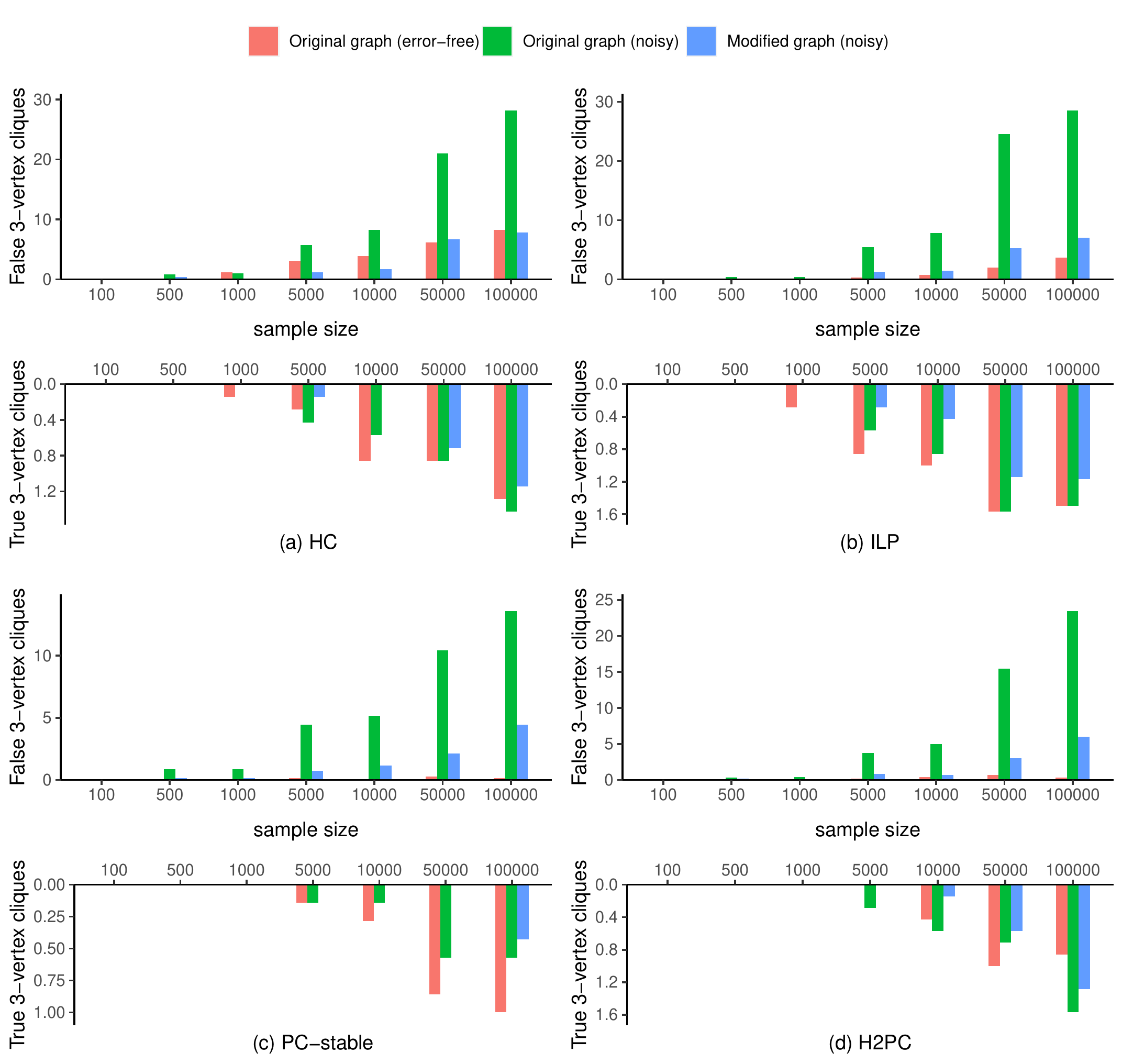}
    \caption{The average number of false and true 3-vertex cliques produced by the original graphs learned from error-free data sets, the original graphs learned from noisy data sets and the modified graphs learned from noisy data sets.}
    \label{fig: 3-vertex cliques}
\end{figure}
\vskip 0.2in
\bibliography{reference}
\end{document}